\documentclass{article} 
\usepackage[preprint]{aistats2026}
\usepackage{times}

\newcommand{\cost}{\mathrm{cost}}
\newcommand{\OPT}{\mathsf{OPT}}
\renewcommand{\epsilon}{\varepsilon}
\usepackage{tikz}

\usepackage{wrapfig}
\usepackage{amsthm,amsfonts,amsmath,amssymb,xcolor,enumerate}
\usepackage[round]{natbib}
\renewcommand{\bibname}{References}
\renewcommand{\bibsection}{\subsubsection*{\bibname}}
\usepackage{hyperref}
\hypersetup{
  colorlinks=true,
  linkcolor={red!60!black},
  citecolor={blue!60!black},
  urlcolor={magenta!60!black}
} 
\hypersetup{colorlinks}
\usepackage{aligned-overset}
\usepackage{caption,subcaption}
\usepackage{algorithmic}
\usepackage{algorithm}
\usepackage{subcaption}

\UseRawInputEncoding
\newtheorem{theorem}{Theorem}[section]
\newtheorem{definition}{Definition}
\newtheorem{proposition}[theorem]{Proposition}
\newtheorem{corollary}[theorem]{Corollary}

\newtheorem{assumption}{Assumption}

\newtheorem{remark}{Remark}
\newtheorem{lemma}[theorem]{Lemma}

\allowdisplaybreaks
\makeatletter
\def\thm@space@setup{%
  \thm@preskip=\parskip \thm@postskip=0pt
}
\makeatother

\title{Differentially Private Wasserstein Barycenters}
\author{
Anming Gu\thanks{Correspondence to \texttt{anminggu@utexas.edu}. Part of this work was done while A.G. was at Boston University.}\\
UT Austin \\
\and 
Sasidhar Kunapuli\\
Independent Researcher\\
\and   
Mark Bun\\
Boston University\\
\and 
Edward Chien\\
Boston University
\and 
Kristjan Greenewald\\
MIT-IBM Watson AI Lab; IBM Research
}
\date{\today}
\begin{document}

\twocolumn
\maketitle
\pagenumbering{arabic}
\pagestyle{plain}
\setlength{\footskip}{20pt}

\begin{abstract}
The Wasserstein barycenter is defined as the mean of a set of probability measures under the optimal transport metric, and has numerous applications spanning machine learning, statistics, and computer graphics. In practice these input measures are empirical distributions built from sensitive datasets, motivating a differentially private (DP) treatment. We present, to our knowledge, the first algorithms for computing Wasserstein barycenters under differential privacy. Empirically, on synthetic data, MNIST, and large-scale U.S. population datasets, our methods produce high-quality private barycenters with strong accuracy-privacy tradeoffs.
\end{abstract}

\section{Introduction}
In the era of big data and machine learning, users are increasingly concerned about their privacy. Differential privacy (DP) \citep{dwork2006dp} has seen widespread adoption to provide guarantees for user privacy. For example, government bureaus use DP when releasing census data \citep{abowd2018us,hod2024dp_israel}, and companies such as Apple \citep{apple_dp}, Microsoft \citep{ding2017microsoft}, and LinkedIn \citep{rogers2020linkedin} extensively employ DP when releasing data --- aiming to protect user data from security threats. 

Clustering, summarizing and reducing the size of datasets are fundamental tasks in unsupervised machine learning. Many of these unsupervised learning problems are NP-hard \citep{megiddo1984complexity,altschuler22nphard}, leading to the development of polynomial-time approximation algorithms \citep{charikar1999constant,charikar1999improved,jain2001approximation,jain2003greedy,charikar2012dependent,cohen2022lsh}. A long line of works \citep{gupta2010dp_comb_opt,balcan2017clustering, kaplan2018differentially,jones2021differentially,chaturvedi2020dp_kmeans_exp_mech,ghazi2020dp_clustering} have further studied clustering under DP, providing polynomial time algorithms with tight approximation bounds.

Defined as the mean of a set of probability measures under the optimal transport metric,\footnote{Specifically, it is the distribution that minimizes the average Wasserstein distance between itself and each distribution in the set, and generalizes the classical concept of the (Euclidean) mean from datapoints to entire distributions.}  the Wasserstein barycenter is a useful notion that contains many of these unsupervised tasks as special cases, with applications to a much more general suite of problems. Specific instances of the Wasserstein barycenter problem include centroids of probability measures \citep{zen2011earth} and $k$-means clustering \citep{canas2012learning}. Consequently, Wasserstein barycenters have seen extensive applications in domain adaptation \citep{montesuma2021wasserstein}, computer graphics \citep{pele2009fast,solomon2015convolutional}, and biology \citep{nadeem2020wb_healthcare,heinemann2022wb_healthcare}.

Similar to clustering and other unsupervised learning problems mentioned above, privatization of Wasserstein barycenters becomes crucial when working with sensitive data. As one of many possible examples, suppose a company wishes to train and deploy machine learning models for analysis of sensitive data for each of many countries. Prior to actually training these models, to minimize the risk of privacy breaches, a machine learning engineer must design the model architecture and tune hyperparameters on a DP synthetic dataset (see \cite{ridgeway2021challenge,mckenna2021winning,xie2018differentially} for more details on this emerging practice), which by the post-processing property of DP can be used and re-used \emph{ad infinitum} without incurring any additional privacy loss. A private Wasserstein barycenter of a (sub)set of country-level datasets would be a natural candidate for this private synthetic dataset, as (1) it averages across many countries and hence should incur much less privacy cost than maintaining separate private synthetic datasets for each country, and (2) it approximately minimizes the Wasserstein distances to the true distributions, which should maximize the chance that the designed model architecture will work well when applied to each country's data at deployment time.

More classically, recall that the Wasserstein barycenter minimizes the (weighted) average transport cost from the barycenter to each marginal. This can be an interesting optimization problem in its own right, e.g. choosing locations for distribution centers for multiple products each with its own geographic demand distribution, where each center has the same mix of products.

Motivated by these considerations, our work answers the following question in the affirmative:

\begin{center}
\textit{Do there exist efficient\footnote{Same (asymptotic) runtime as a non-private algorithm.} algorithms for computing Wasserstein barycenters under DP?}    
\end{center}


\paragraph{Contributions} To the best of our knowledge, we provide the first algorithms for computing Wasserstein barycenters under the constraints of the central model \citep{dwork2006dp} of DP. 
We work under the setting where each individual contributes one datapoint to one distribution.

Our main contributions are as follows.
\begin{itemize}
    \item We provide an efficient $\epsilon$-DP algorithm using a black-box reduction from private Wasserstein distance coresets (Definition \ref{def:coreset_wasserstein_distance}); see Theorem \ref{thm:main}. Here, we form a private version of each distribution, and use these to solve the barycenter problem. 
    \item We provide a lightweight $(\epsilon,\delta)$-DP algorithm that works well when the data is clustered (Definition \ref{def:approx_clusterable}) and the number of support points in the barycenter is not too large; see Theorem \ref{thm:output_perturbation_best}. This method treats the output barycenter as a vector and uses the Gaussian mechanism for privacy.
    \item We show the efficacy of our algorithms when applied to large-scale real-world sensitive data; see Figures \ref{fig:us_population} and \ref{fig:us_population_a}.
\end{itemize}


\section{Preliminaries}\label{sec:preliminaries}
\paragraph{Notation} In big-O notation, we use $\Tilde{O}$ to hide logarithmic factors and subscripts to hide dependence in those variables. We use $a\lesssim_p b$ to denote $a=O_p(b)$, e.g. a constant that depends on $p$. We use $[t]$ to denote the set $\{1,\dots, t\}$. For a function $T:\mathcal{X}\to\mathcal{Y}$ and a measure $\mu$, we use $T_\sharp\mu$ to denote the pushforward measure, e.g. $T_\sharp\mu(B) = \mu(T^{-1}(B))$ for a measurable set $B\subseteq \mathcal{Y}$.  
For a datapoint $x$, we use $\delta_x$ to denote a Dirac delta at $x$. We reserve the Greek letter $\xi$ to denote a failure probability. We use $B_x(R) := \{y\mid \|x-y\|_2\le R\}$ to denote the closed Euclidean ball of radius $R$ centered at point $x$. \textbf{(OT)} We use Greek letters $\mu$ for the dataset and $\nu$ for barycenter, which are probability measures. \textbf{(DP)} We use the script font $\mathcal{A}$ and $\mathcal{D}$ to denote algorithms and datasets, respectively. We use  $\epsilon,\delta$ for privacy parameters for DP. \textbf{(JL)} We use letters $d$ to represent the dimension of the \emph{ambient} space and $d'$ to represent the dimension of the \emph{projected} space from the JL transform. We use the Greek letter $\gamma$ for the multiplicative factor in the JL and $\Pi$ for the projection matrix.

Differential privacy (DP) \citep{dwork2006dp} is a mathematical framework for establishing guarantees on privacy loss of an algorithm, with nice properties such as degradation of privacy loss under composition and robustness to post-processing. We provide a brief introduction and refer to \cite{dwork2014foundations_of_dp} for a thorough treatment.

\begin{definition}[$(\epsilon,\delta)$-DP]
Algorithm $\mathcal{A}$ is said to satisfy $(\epsilon,\delta)$-differential privacy if for all adjacent datasets $\mathcal{D}, \mathcal{D}'$ (datasets differing in at most one element) and all $\mathcal{S} \subseteq \operatorname{range}\mathcal{A}$, it holds \[
\Pr[\mathcal{A}(\mathcal{D})\in \mathcal{S}]\le e^\epsilon\Pr[\mathcal{A}(\mathcal{D}')\in \mathcal{S}]+\delta.
\]
If $\delta = 0$, we drop the dependence on $\delta$ and say $\mathcal{A}$ satisfies $\epsilon$-differential privacy.
\end{definition}

Let $(\mathcal{X}, \rho)$ be a metric space and let $\mathcal{P}(\mathcal{X})$ be the set of Borel probability measures on $\mathcal{X}$. 

\begin{definition}[Wasserstein distance]
    For $p\in [1, \infty)$, the $p$-Wasserstein distance between probability measures $\mu,\nu\in\mathcal{P}(\mathcal{X})$ is defined to be \[
    W_p(\mu,\nu):= \left(\inf_{\pi\in\Pi(\mu,\nu)}\int_{\mathcal{X}\times\mathcal{X}}\rho(x,y)^pd\pi(x,y)\right)^{1/p},
    \]
    where $\Pi(\mu,\nu) := \{\pi \in \mathcal{P}(\mathcal{X}\times\mathcal{X})\mid (P_x)_\sharp \pi = \mu, (P_y)_\sharp \pi = \nu\}$ is the set of transport plans, and $P_x(x,y) := x$ and $P_y(x,y) := y$ are the projections onto the first and second coordinates, respectively.
\end{definition}
We will be using the $L_p$ metric for the cost function, e.g. $\rho(x, y) := \|x - y\|_p$. 

In Appendix \ref{app:lemmata}, we recall some additional facts on differential privacy, optimal transport, and the Johnson-Lindenstrauss transform.

\section{Problem statement}\label{sec:problem_statement}
\cite{agueh2011barycenters} introduced the notion of barycenters on Wasserstein space: 
\begin{definition}[Wasserstein barycenter]
   Given probability distributions $\mu_1,\dots, \mu_k\in\mathcal{P}(\mathcal{X})$ and weights $\beta_1,\dots, \beta_k > 0$, the $p$-Wasserstein barycenter is any distribution $\nu^\ast$ satisfying 
   \begin{equation}
      \nu^\ast \in \underset{\nu\in\mathcal{P}(\mathcal{X})}{\arg\min}\sum_{i=1}^k \beta_i W_p^p(\mu_i, \nu).\label{eq:wb_objective}
   \end{equation}
\end{definition}

We will be working with \emph{discrete distributions}, where each distribution can be thought of as a subpopulation, and one individual contributes sensitive data to one of the distributions. For instance each of these distributions could represent the data from one country.

Formally, we have $k$ empirical distributions $\mu_i$ for $i \in [k]$, each with $n$ point masses, where 
\begin{equation}\label{eq:mu_i}
\mu_i = \frac{1}{n}\sum_{j=1}^n \delta_{x_j},  
\end{equation}

Our goal is to compute a distribution $\nu$ consisting of exactly $m\le n$ point masses and uniform weights that minimizes the objective \eqref{eq:wb_objective} under the constraints of DP. For any application of DP, a definition of neighboring datasets is required. We use the following.
\begin{definition}[Neighboring datasets]
Let $\{\mu_i\}_{i\in[k]}$ be a collection of empirical distributions. Let the dataset be $\mathcal{D} := \{(x_{i,j},i)\mid  x_{i,j}\in \mu_i, i\in [k]\}$. We say two datasets $\mathcal{D},\mathcal{D}'$ are neighboring if they differ by exactly one element in the first coordinate, e.g. they differ by one point in one distribution. 
\end{definition}
This definition of neighboring datasets is motivated by viewing each $\mu_i$ as a non-overlapping subpopulation, i.e. we are essentially assuming that individuals do not appear in multiple of the $\mu_i$. Without loss of generality, we can assume all of the points are distinct. We will also abuse notation and identify $\mathcal{D}$ with $\{\mu_1,\dots, \mu_k\}$.

In order to limit the influence of any individual, we require an assumption on the support. For simplicity, without loss of generality we assume a support contained in a ball of radius of 1/2.
\begin{assumption}\label{asmp:supp}
It holds that $\cup_{i\in[k]}\operatorname{supp} \mu_i\subseteq B_0(1/2)$.
\end{assumption}
This is a standard assumption in private clustering, e.g. see \cite{ghazi2020dp_clustering},\footnote{\cite{ghazi2020dp_clustering} considers the ball of radius $1$, while we consider the ball with diameter $1$.} and more generally private convex optimization. This assumption is used to simplify the description of the results as it is well known that the additive error of any DP algorithm scales proportionally with respect to the radius of the support of the dataset, e.g. see \cite{altschuler2024dpsgd}.

\cite{izzo2021wasserstein_barycenter} utilizes the Johnson-Lindenstrauss transform to speed up algorithms for Wasserstein barycenters. We start with the following definition, adapted from Definition 2.1 of \cite{izzo2021wasserstein_barycenter}.\footnote{Their definition is slightly different. Our definition is simplified to work better with Assumption \ref{asmp:label_weights}.} 
\begin{definition}[Solution]\label{def:solution} 
    Fix a candidate barycenter $\nu$ supported on points $\nu^{(1)}, \dots, \nu^{(m)}$. Define the solution $(\mathbf{S}, \mathbf{w}) := (S_1, \dots, S_m, w_1, \dots, w_m)$ as follows. Let $w_j(x)$ denote the total weight transported from $x\in \cup_{i=1}^k\operatorname{supp}\mu_i$ to point $\nu^{(j)}$ based on the optimal transport plan. 

    Define the set  \[
S_j := \left\{x\in \mathcal{D}\,\bigg|\, w_j(x)>0\right\}.
    \]
\end{definition}
We call $(\mathbf{S}, \mathbf{w})$ a solution because the following holds. For each $j\in [m]$, $\nu^{(j)}$ minimizes the objective
\begin{align}
    \sum_{x\in S_j} w_j(x) \|x-\nu^{(j)}\|^p.
\end{align}

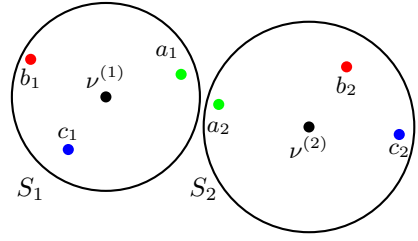
\begin{figure}
    \centering
    \begin{tikzpicture}
\draw[thick] (0,0) circle (1.25cm); 
\node[anchor=south] at (-1,-1.5) {$S_1$};
\fill[black] (0,0) circle (2.1pt);
\node[anchor=south] at (0,0) {\small $\nu^{(1)}$};
\fill[red] (-1,0.5) circle (2.1pt); 
\node[anchor=north] at (-1,0.5) {\small $b_1$};
\fill[blue] (-0.5,-0.7) circle (2.1pt); 
\node[anchor=south] at (-0.5,-0.7)  {\small $c_1$};
\fill[green] (1.5,-0.1) circle (2.1pt); %
\node[anchor=north] at (1.5,-0.2) {\small $a_2$};
\node[anchor=east] at (1.1,0.6) {\small $a_1$};
\fill[green] (1,0.3) circle (2.1pt); 

\draw[thick] (2.7,-0.4) circle (1.4cm); 
\node[anchor=south] at (1.3,-1.5) {$S_2$};
\fill[black] (2.7,-0.4) circle (2.1pt);
\node[anchor=north] at (2.7,-0.4) {\small $\nu^{(2)}$};
\fill[red] (3.2,0.4) circle (2.1pt); 
\node[anchor=north] at (3.2,0.4) {\small $b_2$};
\fill[blue] (3.9,-0.5) circle (2.1pt); 
\node[anchor=north] at (3.9,-0.5) {\small $c_2$};

\end{tikzpicture}
    \caption{Example of a solution. The input distributions are $\mu_a := \frac{1}{2}\delta_{a_1} + \frac{1}{2}\delta_{a_2}, \mu_b := \frac{1}{2}\delta_{b_1} + \frac{1}{2}\delta_{b_2}, \mu_c := \frac{1}{2}\delta_{c_1} + \frac{1}{2}\delta_{c_2}$ and the candidate barycenter is $\nu := \frac{1}{2}\delta_{\nu^{(1)}} + \frac{1}{2}\delta_{\nu^{(2)}}$. Observe that: $S_1 = \{a_1, b_1, c_1\}, S_2 = \{a_2, b_2, c_2\}, w_1(a_1) = w_1(b_1) = w_1(c_1)= w_2(a_2) = w_2(b_2) = w_2(c_2) = 1$.}
    \label{fig:solution}
\end{figure}

See Figure \ref{fig:solution} for intuition on the definition of this solution. Notice that if we are given the weights $w_j$, we can  easily reconstruct the points $\nu^{(j)}$ using convex optimization. We obtain these weights by solving a corresponding Wasserstein barycenter in the reduced space using any approximation algorithm. Note that by conservation of mass, it holds that $\sum_{j=1}^nw_j(x) = 1$.

We constrain the optimization objective in \eqref{eq:wb_objective} as follows.
\begin{assumption}\label{asmp:label_weights}
    Assume that the objective \eqref{eq:wb_objective} has an added constraint that the solution has $m$ equally weighted atoms, where $m$ is specified. Specifically, the solution satisfies $
    \nu = \frac{1}{m}\sum_{j=1}^m \delta_{\nu^{(j)}},
    $
    where $m \le n$, and $\beta_i = \frac{1}{k}$.\footnote{We remark that our first algorithm (Theorem \ref{thm:main}) does not require any assumption on $\beta_i$.}
\end{assumption} 
We make two remarks on the uniform weight assumption:
\begin{itemize}
    \item From an \emph{interpretability} perspective, uniform weights is a reasonable assumption so that each datapoint can be considered as data representing one synthetic person.
    \item From a \emph{computational} perspective, many papers on Wasserstein barycenters \emph{a priori} solve the problem under the uniform weight assumption, as optimizing weights for barycenters is much more challenging than optimizing for supports, e.g. see the discussion in \cite{altschuler21barycenter}.
\end{itemize}

We will use the following cost function for Wasserstein barycenters.
\begin{definition}[Cost]
    For a solution $(\mathbf{S}, \mathbf{w})$, define its cost to be the value of \eqref{eq:wb_objective} when $\nu$ is reconstructed from $(\mathbf{S}, \mathbf{w})$: \begin{equation}
      \operatorname{cost}(\mathbf{S}) := \min_\nu \frac{1}{nk}\sum_{j=1}^n \sum_{x\in S_j}w_j(x) \|x-\nu^{(j)}\|^p.
\label{eq:cost_of_soln}      
    \end{equation}

Similarly, for a projection $\Pi$, define $\operatorname{cost}(\Pi_\sharp \mathbf{S})$ to be the value of \eqref{eq:wb_objective} when we first project each distribution to $\mathbb{R}^{d'}$ using $\Pi$, then compute $\Tilde{\nu}$ using the original weights $w_j$:
\begin{equation}
    \operatorname{cost}(\Pi_\sharp \mathbf{S}) := \min_\nu \frac{1}{nk}\sum_{j=1}^n \sum_{x\in S_j}w_j(x) \|\Pi x-\Tilde{\nu}^{(j)}\|^p.\label{eq:cost_of_jl_soln}
\end{equation}
\end{definition}
Above, note that we suppress the dependence on $p$ for the cost.
We use the following definition for approximate Wasserstein barycenters.
\begin{definition}[Approximate Wasserstein barycenter]\label{def:approx_wb}
Let $\mathsf{OPT}$ be the minimum of \eqref{eq:wb_objective}. A $(z,t)$-approximation for the $p$-Wasserstein barycenter is probability measure $\nu$ such that \[
\operatorname{cost}(\nu)\le z \cdot \mathsf{OPT}_{(\mu_1,\dots,\mu_k)} + t,\]
where $\mathsf{OPT}$ is the cost of an optimal barycenter supported on $m$ atoms with uniform weights. When it's clear, we suppress the dependence on the input barycenters.
\end{definition}

\section{A private coreset approach}\label{sec:coreset}
In this section, we use a black-box private coreset approach. In particular, we consider private distributions that are close in Wasserstein distance to the sensitive distribution. We use these coresets as input to the approximate Wasserstein barycenter algorithm to obtain the private barycenter via post-processing.

We start by introducing the notion of coresets for Wasserstein distance.

\begin{definition}[Coreset for Wasserstein distance]\label{def:coreset_wasserstein_distance}
    A measure $\mu'$ is a $(p, z, t)$-coreset of $\mu$ for the $p$-Wasserstein distance if for every $\pi\in\mathcal{P}(\mathbb{R}^d)$, we have $W_p(\mu',\pi)\le z\cdot W_p(\mu,\pi) + t$. When $p$ is unambiguous, we drop the $p$. 
\end{definition}

The following proposition is a direct consequence of the triangle inequality.
\begin{proposition}
    If $W_p(\mu,\mu')\le t$, then $\mu'$ is a $(p, 1, t)$-coreset of $\mu$ for the $p$-Wasserstein distance problem.
\end{proposition}

Now our goal is to find a coreset for the $p$-Wasserstein distance problem. We use the algorithm from \cite{he2023algorithmically}. Informally, the algorithm works as follows. First, obtain a hierarchical binary partition over the space of $\log \epsilon n$ levels. Use the (discrete) Laplace mechanism on each cell to compute the number of points in each cell, with noise calibrated to the level. Then, it suffices to choose points in each cell totaling the number of counts independently of the data. The set of all of these points becomes the private data. For a full description of the algorithm, see Algorithm 4 of \cite{he2023algorithmically}.

We remark that there exists a corresponding with high probability (w.h.p.) algorithm that has the following guarantee. 

\begin{theorem}\label{thm:wasserstein_distance_coreset}
    For every $\epsilon > 0$ and $\xi\in(0,1)$, there exists an $\epsilon$-DP algorithm running in time $\Tilde{O}(\epsilon dn)$ that with probability $1-\xi$, outputs an \[\left(p, 1,O_p\left(\left(\frac{1}{(\epsilon n)^{1/d}}\cdot\operatorname{poly}\log\left(\frac{1}{ \xi}\right)\right)^{1/p}\right)\right)\]-approximate coreset of size $O(n\log\epsilon n)$ for the Wasserstein distance.
\end{theorem}

\begin{proof}
    \cite{he2023algorithmically} provides an algorithm with guarantees in expectation for the $W_1$ distance. Due to Lemma \ref{lem:reverse_comp}, this implies a similar guarantee for $W_p$ distance. To obtain the w.h.p. algorithm, we run $O\left(\log\frac{1}{\xi}\right)$ trials of the algorithm and use the exponential mechanism \citep{mcsherry2007exp_mech} to choose the best one, e.g. see Appendix D of \cite{bassily2014private} for an analogous argument.
\end{proof}


Our key technical lemma is the following result to bound the error using Wasserstein distance coresets instead of the true distributions:
\begin{lemma}\label{lem:wass_distance_to_coreset}
    Let $\mu_1,\dots, \mu_k$ be discrete probability measures and suppose $\mu_1',\dots, \mu_k'$ are $(p, 1,t)$-coresets for each $\mu_i$, respectively. Then,
    \[
    \OPT_{(\mu_1',\dots, \mu_k')}\le  \OPT_{(\mu_1,\dots,\mu_k)}+O_p(t^p).
    \]
\end{lemma}
\begin{proof}[Proof (sketch)]
    This follows from Definition \ref{def:coreset_wasserstein_distance} and \eqref{eq:cost_of_soln}. See Appendix \ref{app:lem_wass_distance_to_coreset} for the full proof.
\end{proof}

\cite{izzo2021wasserstein_barycenter} generalized the breakthrough work of \cite{makarychev2019performance} to show that reducing to $O(\log n)$ dimension suffices to preserve the cost of $p$-Wasserstein distances for \emph{all} solutions supported on at most $n$ data points. Their main result is the following:

\begin{theorem}\label{thm:izzo_main}
    Let $\mu_1, \dots, \mu_k$ be discrete probability distributions on $\mathbb{R}^d$ such that $|{\operatorname{supp}\mu_i}|\le \operatorname{poly}(n)$ for all $i \in [k]$. Let $d\ge 1$, $\gamma,\xi\in (0, 1)$. Let $\Pi:\mathbb{R}^d\to \mathbb{R}^{d'}$ be an i.i.d. Gaussian JL map with $d' = O\left(\frac{p^4}{\gamma^2}\log \frac{n}{\gamma\xi}\right)$. 
    Then, with probability $1 - \xi$, it holds that 
    \[
    \operatorname{cost}(\mathbf{S}) \approx_{1+\gamma}\operatorname{cost}(\Pi_\sharp \mathbf{S})
    \]
    for all solutions $(\mathbf{S},\mathbf{w})$.
\end{theorem}
Above, for $\gamma\ge 0$, we use $a\approx_{1+\gamma}b$ to denote $\frac{1}{1+\gamma}\le \frac{a}{b} \le 1 + \gamma$. We briefly remark that it is also possible to use the fast JL transform using  $d' = O\left(\frac{p^6}{\gamma^2}\log \frac{n}{\gamma\xi}\right)$; for details please refer to Appendix B of \cite{izzo2021wasserstein_barycenter}.

Our main result in this section is the following, whose proof we provide in Appendix \ref{app:thm_main}.
\begin{algorithm}[t]
\begin{algorithmic}[1]
\renewcommand{\COMMENT}[1]{%
  \hfill\makebox{$\triangleright$~~#1}}

\REQUIRE{$k$ discrete distributions $\mu_1,\dots, \mu_k$ supported on $\mathbb{R}^d$, projection dimension $d'$, approximate Wasserstein barycenter algorithm $\mathsf{Approx}$, privacy parameter $\epsilon$}
\STATE{Sample a JL transform $\Pi \in\mathbb{R}^{d\times d'}$}
\FOR{$i\in[k]$}
\STATE$\mu'_i\gets \mathsf{WassersteinDistanceCoreset}^{\epsilon}(\mu_i)$
\STATE{$\Hat{\mu}_i \gets \Pi_\sharp \mu_i'$}
\ENDFOR
\STATE{$\Hat{\nu}\gets \mathsf{Approx}(\Hat{\mu}_1,$ $\dots, \Hat{\mu}_k)$}\COMMENT{$\hat{\nu}\in \mathbb{R}^{d'}$}
\STATE{$(\mathbf{S}, \mathbf{w})\gets \mathsf{SolutionWeights}(\Hat{\nu}, \Hat{\mu}_1, \dots, \Hat{\mu}_k)$}
\STATE{$(\nu^{(1)},\dots, \nu^{(m)})\gets \mathsf{SupportPoints}(\mu_1',\dots, \mu_k',$ $\mathbf{S},\mathbf{w})$}
\RETURN{$\nu$ with uniform support on $\{\nu^{(j)}\}_{j\in[m]}$}
\end{algorithmic}
\caption{$\mathsf{CoresetBarycenter}^{\epsilon}$}
\label{alg:wb_jl}
\end{algorithm}

\begin{theorem}\label{thm:main}
    For any $p\ge 1$, suppose that there exists a (not necessarily private) $(z, t)$-approximation algorithm that runs in time $2^{O(d)}\cdot \operatorname{poly}(n,k)$ for the $p$-Wasserstein barycenter problem. Then, for every $\epsilon>0$ and $ \gamma,\xi\in(0,1)$, there exists a polynomial-time $\epsilon$-DP algorithm that outputs an 
    \[
    \left(z(1+\gamma), \Tilde{O}_{p,\gamma,z,\xi}\left(\frac{1}{(\epsilon n)^{1/d}} + t\right)\right)\]
    -approximate $p$-Wasserstein barycenter, with probability $1-\xi$.
\end{theorem}

Here, we suppress poly-log dependence on $1/\xi$. Via dimensionality reduction, we can afford an algorithm that has an \emph{exponential dependence} on the dimension as $d' = O(\log n)$. Unfortunately, many state of the art additive approximation algorithms still do not lend polynomial runtime when combined with dimensionality reduction. For instance, the algorithm of \cite{altschuler21barycenter} runs in time $(nk)^{O(d)}$.

The weights $(\mathbf{S}, \mathbf{w})$ from Definition \ref{def:solution} are computed via optimal transport plans between $\hat{\mu}_i$ and $\hat{\nu}$, e.g. the distributions in low dimension. Due to post-processing, the $\hat{\mu}_i$ are private, so the computation incurs no additional privacy loss. 
We provide pseudocode in Algorithm \ref{alg:private_weights}.
To recover the support points, we use empirical risk minimization (Algorithm \ref{alg:find_support}).


\section{An output perturbation approach}\label{sec:basic_output_perturbation}
One issue of the method in the previous section is that the curse of dimensionality implies that if $d = \Omega(\log n)$, then the additive error becomes a constant. To address this, we can alternately consider output perturbation. 

However, a naive application of output perturbation will only yield good utility if $md \ll (\epsilon k)^2$. The issue that inhibits an upper bound that benefits from increasing $n$ is that in a neighboring dataset, the couplings of all $n$ points in the updated distribution could potentially change, so we only obtain an averaging effect due to the $k$ distributions, as opposed to $\frac{nk}{m}$ (number of points that are mapped to each point in the support of the barycenter); see Proposition \ref{prop:counterexample2}.

\begin{algorithm}[t]
\begin{algorithmic}[1]
\renewcommand{\COMMENT}[1]{%
  \hfill\makebox{$\triangleright$~~#1}}

\REQUIRE{$k$ discrete distributions $\mu_1,\dots, \mu_k$ supported on $\mathbb{R}^d$, approximate Wasserstein barycenter algorithm $\mathsf{Approx}$, privacy parameter $\epsilon, \delta$}
\STATE{$\nu:=(\nu^{(1)}, \dots, \nu^{(m)})\gets \mathsf{Approx}^{\mathsf{JL}}(\mu_1,\dots, \mu_k)$}
\STATE{$(\Tilde{\nu}^{(1)}, \dots, \Tilde{\nu}^{(m)})\gets (\nu^{(1)}, \dots, \nu^{(m)}) + \mathcal{N}(0, \sigma^2\cdot I_{md})$, where \[
\sigma^2 := \frac{2m\log(1.25/\delta)}{\left(\epsilon k\right)^2}. 
\]}
\RETURN{$\Tilde{\nu}$ with uniform support on $\{\tilde{\nu}^{(j)}\}_{j\in[m]}$}
\end{algorithmic}
\caption{$\mathsf{OutputPerturbationBarycenter}^{\epsilon,\delta}$}
\label{alg:wb_output_perturbation}
\end{algorithm}

We provide the pseudocode for this basic output perturbation method in Algorithm \ref{alg:wb_output_perturbation}. Here, $\mathsf{Approx}^\mathsf{JL}$ denotes the algorithm with dimensionality reduction, as in Section \ref{sec:coreset}. Its guarantees are as follows, and the proof is provided in Appendix \ref{app:thm_wb_output_perturbation}.

\begin{theorem}\label{prop:wb_output_perturbation}
 For any $p\ge 1$, suppose that there exists a (not necessarily private) $(z, t)$-approximation algorithm that runs in time $2^{O(d)}\cdot \operatorname{poly}(n,k)$ for the $p$-Wasserstein barycenter problem. Then, for every $\epsilon>0$ and $\delta,\gamma,\xi \in (0, 1)$, there exists a polynomial-time $(\epsilon,\delta)$-DP algorithm that outputs an \[
\left(z(1+\gamma), \Tilde{O}_{p,\gamma,z,\xi}\left(\left(\frac{md\log (1/\delta)}{(\epsilon k)^2}\right)^{p/2}\right) + t\right)\]
-approximation for the $p$-Wasserstein barycenter problem, with probability $1 - \xi$. 
\end{theorem}

Fortunately, by splitting each distribution $\mu_i$ into $k'$ disjoint distributions $\{\mu_{i,j}\}_{j\in[m]}$, we can obtain benefits from increasing $n$. 
\begin{proposition}\label{thm:subsampled_output_perturbation}
By splitting each distribution into $k'$ disjoint distributions, in the same setting as Theorem \ref{prop:wb_output_perturbation}, we have an
\begin{align*}
    &\bigg(z(1+\gamma), \Tilde{O}_{p,\gamma,z,\xi}\bigg(\left(\frac{md\log (1/\delta)}{(\epsilon kk')^2}\right)^{p/2} \\
    &\qquad\qquad\qquad\qquad + \left(1 - \frac{1}{k'}\right) + t\bigg)\bigg)
\end{align*}
-approximation for the $p$-Wasserstein barycenter problem, with probability $1 - \xi$.
\end{proposition}

The proof and pseudocode are provided in Appendices \ref{app:thm_subsampled_output_perturbation} and \ref{app:deferred_algs}, respectively. Note that the additive error is asymptotically vacuous because the solution $\delta_0$ has cost $O(1)$. Nonetheless, this $1-\frac{1}{k'}$ factor is for worst-case data; for real-world settings, data are clustered, e.g. see Figures \ref{fig:us_population} and \ref{fig:us_population_a}. To make precise the definition of clustered distributions, we consider a slight modification of the definition of  \cite{weed2019asymptotic,solomon2022k,greenewald2021k}. 

\begin{definition}[$(m,\Delta, c)$-approximately clusterable distribution]\label{def:approx_clusterable}
    A distribution $\mu$ is $(m,\Delta,c)$-approximately clusterable if at least a $1-\epsilon$ fraction of $\mathrm{supp}(\mu)$ lies in the union of $m$ balls of radius at most $\Delta$.
\end{definition}

We prove an empirical rate of convergence bound for such distributions.
\begin{proposition}[Informal, see Proposition \ref{prop:approx_clusterable'}]\label{prop:approx_clusterable}
If $\mu$ is $(m,\Delta, c\le \frac{1}{2})$-approximately clusterable, then for all $n \le m(2\Delta)^{-2p}$, letting $\hat{\mu}_n = \frac{1}{n}\sum_{i\in[n]}\delta_{X_i}$ for $X_i\sim \mu$ i.i.d., it holds \begin{align*}
    &W_p^p(\mu, \hat{\mu}_n)\lesssim_{p} (1-c)\sqrt{\frac{m}{n}}+c + \sqrt{\frac{\log(1/\xi)}{n}}.
    \end{align*}
    with probability $ 1-\xi$.
\end{proposition}
Essentially, the first term on the right-hand-side is from the clustered portion \citep{weed2019asymptotic}, the second is the unclustered portion, and the last is a bound on the contributions from mismatched items. We provide the proof in Appendix \ref{app:prop_approx_clusterable}.

\begin{figure}[t]
    \centering
    \begin{subfigure}[t]{0.23\textwidth}
    \centering
    \includegraphics[width=\textwidth]{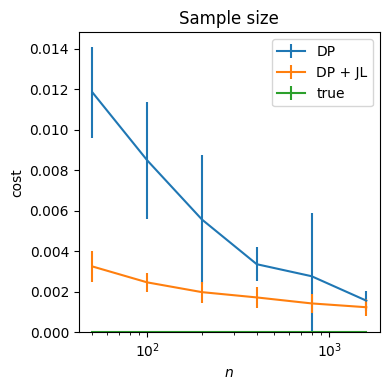}
    \caption{$n = 50, 100, \dots, 1600$.}
    \end{subfigure}
    \begin{subfigure}[t]{0.23\textwidth}
    \centering
    \includegraphics[width=\textwidth]{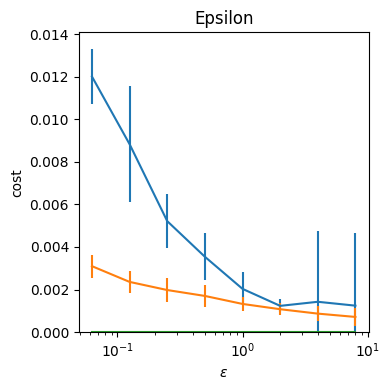}
    \caption{$\epsilon = 2^{-4}, 2^{-2}, \dots, 2^3$.}
    \end{subfigure}\\
    \begin{subfigure}[t]{0.23\textwidth}
    \centering
    \includegraphics[width=\textwidth]{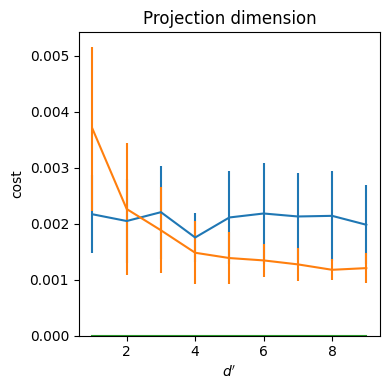}
    \caption{$d' = 1, 2, \dots, 9$.}
    \end{subfigure}
    \caption{Synthetic experiments testing sample size $n$, privacy parameter $\epsilon$, and projection dimension $d'$, averaged over 30 runs for the private coreset approach of Section \ref{sec:coreset}.}
    \label{fig:additional}
\end{figure}

Using this proposition, we can prove the following theorem. 

\begin{theorem}\label{thm:output_perturbation_best}
    Suppose every input distribution $\{\mu_i\}_{i\in[k]}$ is $\left(O(m), \Delta, c\le \frac{1}{2}\right)$-approximately clusterable and $\frac{n}{k'}\le O(m)\cdot (2\Delta)^{-2p}$.  For any $p\ge 1$, suppose that there exists a (not necessarily private) $(z, t)$-approximation algorithm that runs in time $2^{O(d)}\cdot \operatorname{poly}(n,k)$ for the $p$-Wasserstein barycenter problem. Then, for every $\epsilon>0$ and $\delta,\gamma,\xi \in (0, 1)$, there exists an $(\epsilon,\delta)$-DP algorithm that outputs a \begin{align*}
          & \bigg(z(1+\gamma), \Tilde{O}_{p,\gamma,z,\xi}\bigg(\left(\frac{md\log (1/\delta)}{(\epsilon kk')^2}\right)^{p/2} \\
        &\qquad\qquad\qquad\qquad\quad+{}\sqrt{\frac{mk'}{n}}+c + t\bigg)\bigg)
    \end{align*}
 
    -approximation for the $p$-Wasserstein barycenter problem, with probability $1-\xi$.
\end{theorem}
Here, the algorithm is the same as that of Proposition \ref{thm:subsampled_output_perturbation}, Algorithm \ref{alg:wb_output_perturbation_subsampled}; and the proof is provided in Appendix \ref{app:thm_output_perturbation_best}. Optimizing the first two terms so they are equal, suppressing dependence on $\delta$, we should set $k'$ to be \[
k'_* = \Theta\left(n^{\frac{1}{2p+1}}m^{\frac{p-1}{2p+1}}d^{\frac{p}{2p+1}}(\epsilon k)^{-\frac{2p}{2p+1}}\right).
\] 
Using this $k_*'$, by suppressing poly-log dependence on $\delta$ and $\xi$, this yields an excess asymptotic additive error of \[
m^\frac{3p}{4p+2}d^{\frac{p}{4p+2}}(\epsilon kn)^{-\frac{p}{2p+1}} + c.
\]
Thus, we will have good utility if $m^3d\ll (\epsilon kn)^2$ and $c$ is sufficiently small. 

Under the assumption that the data is approximately clustered and $c$ is small, e.g. $O(n^{-\frac{1}{2}})$, we should always pick this approach over the private coreset approach for small to moderate regimes of $m$. Further, observe that this method does not suffer from the curse of dimensionality as the scaling is with respect to $\approx (nk)^{-\frac{1}{2}}$ as opposed to $n^{-\frac{1}{d}}$ in the coreset-based approach.

\section{Experiments}\label{sec:experiments}

\begin{figure*}[ht!]
    \centering
    \begin{subfigure}{\textwidth}
    \centering
    \includegraphics[width=0.95\textwidth]{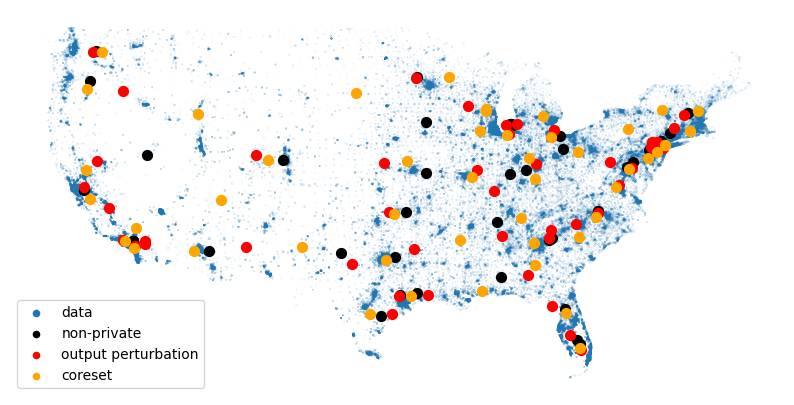}
    \caption{$n = 200000$ and $\epsilon = 1$ (and $\delta = \frac{1}{n}$) for $m = 48$ and $k = 1$. Denoting $\nu,\nu_\mathsf{core},\nu_\mathsf{pert}$ as the non-private, private coreset-based, and output-perturbation barycenters, respectively, we have $\mathrm{cost}(\nu) = 15.92$, $\mathrm{cost}(\nu_\mathsf{core}) = 21.62$, $\mathrm{cost}(\nu_\mathsf{pert}) = 16.031$ (squared degrees longitude/latitude),
    and $W_2(\nu,\nu_\mathsf{core}) = 5.633$, $W_2(\nu,\nu_\mathsf{pert}) = 2.665$ (degrees).}\label{fig:us_k=1}
    \end{subfigure}\\
    \begin{subfigure}{\textwidth}
    \centering
    \includegraphics[width=0.95\textwidth]{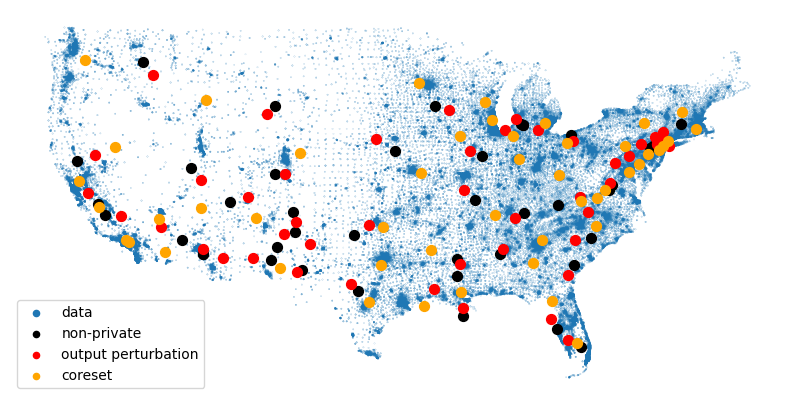}
    \caption{$n = 100000$ and $\epsilon = 1$ (and $\delta=\frac{1}{n}$) for $m = 48$ and $k = 4$ (self-reported White, Asian, Black, Hispanic). Denoting $\nu,\nu_\mathsf{core},\nu_\mathsf{pert}$ as the non-private, private coreset-based, and output-perturbation barycenters, respectively, we have $\mathrm{cost}(\nu) = 5018.618$, $\mathrm{cost}(\nu_\mathsf{core}) = 6770.353$, $\mathrm{cost}(\nu_\mathsf{pert}) = 4924.467$ (squared degrees longitude/latitude),
    and $W_2(\nu,\nu_\mathsf{core}) = 11.978$, $W_2(\nu,\nu_\mathsf{pert}) = 1.622$ (degrees).}
    \label{fig:us_k=4}
    \end{subfigure}
    \caption{Barycenters on continental US populations.}
    \label{fig:us_population}
    \vspace{-0.15in}
\end{figure*}

\begin{figure*}[t]
    \centering
    \includegraphics[width=0.6\linewidth]{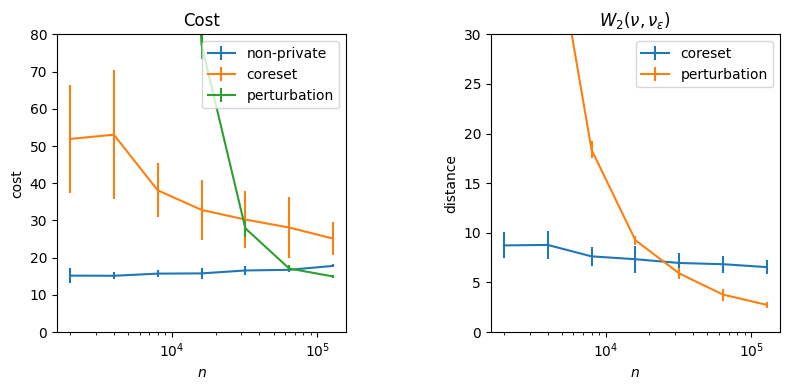}
    \caption{$n = 2000, 4000, \dots, 128000$ and $\epsilon = 1$ (and $\delta = \frac{1}{n}$) in the same experimental setup as Figure \ref{fig:us_k=1}, averaged over 10 trials. On the left, we have cost in squared degrees. On the right, we plot the 2-Wasserstein distance between the private and non-private barycenters (in degrees).}
    \label{fig:us_population_a}
\end{figure*}

We test our method from Section \ref{sec:coreset} on simple synthetic data and MNIST, and both methods on the US population data. We provide additional experiments and discussion of setup in Appendix \ref{app:additional_experiments}.
All of our experiments use the Sinkhorn free support barycenter \citep{flamary2021pot} with $50$ iterations and $100$ inner (Sinkhorn) iterations. We utilize subsampling of the private coresets, which increases cost by a negligible amount, but significantly improves the runtime (see Figure \ref{fig:mnist} on MNIST in Appendix \ref{app:additional_experiments}).

In Figure \ref{fig:additional}, we consider equally weighted mixtures of 4 Gaussians at $(\pm0.25, \pm0.25, 0, \cdots, 0)\in\mathbb{R}^{10}$. We use $m = 8$ and $0.04$ for the entropic regularization. We fix $n = 1000$, $\epsilon = 1$, and $d' = 5$ (when they are not varied). We observe that using JL provides better utility (for small $n$ or $\epsilon$) under DP, as otherwise the algorithm tends to get stuck in local minima.

In Figure \ref{fig:us_population}, we consider the experiment setup from \cite{cuturi2014_wb} ($m = 48$, $k = 1$) and use US population data from the 2015 American Community Survey.\footnote{https://www.census.gov/acs/www/data/data-tables-and-tools/data-profiles/2015/} In Figure \ref{fig:us_k=1}, we take $k = 1$, where the dataset is the whole US population. We take the (sensitive) data to be multisets of the centers of census tracts (chosen with replacement) of size $n = 200000$. In Figure \ref{fig:us_k=4}, we take the sensitive data to be each of $n = 100000$ (uniformly at randomly chosen) points corresponding to the self-reported racial groups White, Asian, Black, Hispanic, where for privacy, we assume the groups are disjoint.

In our private coreset construction, we only sample points that are inside the US border (which is private by data independent post-processing). Our algorithms use $\epsilon = 1$ on the \emph{full} population (or subgroup), utilizing privacy amplification by subsampling \citep{balle2018subsampling}. Each barycenter computation only take a few minutes to run on CPU; however, the sampling of points inside the US takes a few hours for the largest experiments. In this experiment, we use entropic regularization of $0.001$ and do not use dimensionality reduction. 

For our output perturbation algorithms, we always take $k'=1000$, and we use $(\epsilon =1, \delta = \frac{1}{n})$-DP on the \emph{full} population, again using privacy amplification by subsampling.

In Figure \ref{fig:us_population_a}, we report results for the $k = 1$ setting for the US population experiments as the non-private cost in the $k \ge 2$ setting is already very large. In Appendix \ref{app:additional_experiments}, we provide additional experiments with $\epsilon =5$. We empirically observe that, due to the clustered-nature of the US population, subsampling becomes competitive with the coreset construction when each subdistribution has just 32 datapoints! For larger $n$, we hypothesize the smaller cost for the subsampling $k'$ compared to the full non-private barycenter is a result of optimization issues. As the (full) distribution consists of multisets of points, it is more likely for optimization issues to occur (e.g. see the average error for the non-private barycenter is increasing), and subsampling into multiple distributions makes the optimization more stable.

\section{Conclusion}\label{sec:conclusion}
We extended the study of private facility location problems from clustering to Wasserstein barycenters. One limitation of Algorithm \ref{alg:wb_jl} is the curse of dimensionality, and future work can study the setting where the data lies near a low dimensional subspace \citep{weed2019asymptotic} and alleviate the curse of dimensionality via privatized versions of entropic OT \citep{mena2019statistical,genevay2019sample} or  Gaussian-smoothed OT \citep{goldfeld2020convergence,goldfeld2020gaussian,nietert2021smooth,zhang2021gaussian_smoothed_ot}. 
In this work, we studied Wasserstein barycenters under the central model of differential privacy, for future work it would also be interesting to obtain results under the local \citep{kasiviswanathan08ldp} and shuffle \citep{bittau2017shuffle,erlingsson2019shuffle,cheu2019shuffle} models. 
Furthermore, we focused on the setting where one individual contributes a single datapoint out of the $k$ distributions. An interesting direction would be to consider the setting where one person contributes a whole probability measure, as this would allow practitioners to consider \emph{continuous} distributions as input data. 


\bibliography{references}
\bibliographystyle{iclr2026_conference}

\newpage
\onecolumn

\appendix


\section{Related work}
In the theoretical computer science community, the Wasserstein barycenter falls under the category of facility location problems. This class of problem is concerned with placing points, or ``facilities,'' to minimize some objective given a set of input data. Note that clustering also falls under this category. Clustering \citep{lloyd1982least} has seen many non-private approximation algorithms. Over the past few decades, a line of works \citep{charikar1999constant,charikar1999improved,jain2001approximation,jain2003greedy,charikar2012dependent,cohen2022lsh} have pushed multiplicative approximation factors to $2.406$ and $5.912$ for Euclidean $k$-medians and $k$-means, respectively \citep{cohen2022improved}. 

\cite{gupta2010dp_comb_opt} initiated the study of facility location algorithms under DP, and provided an inefficient algorithm based on the exponential mechanism \citep{mcsherry2007exp_mech} that gave constant factor multiplicative approximation. 
Then a series of works \citep{balcan2017clustering, kaplan2018differentially,jones2021differentially,chaturvedi2020dp_kmeans_exp_mech,ghazi2020dp_clustering} culminated in polynomial time algorithms for private clustering with the optimal multiplicative approximation ratio and small additive errors. 

On the other hand, the Wasserstein barycenter is a much more nascent problem. Initial works provide approximations using methods such as entropic regularization \citep{cuturi2013sinkhorn,cuturi2014_wb}, iterative Bregman projections \citep{benamou2014_bregman}, or stochastic optimization \citep{claici2018stochasticwassersteinbarycenters}; however, these lack worst-case guarantees on the approximations, for instance to the non-entropic setting. Even theoretical guarantees for fast approximations 
of Wasserstein distances are recent \citep{altschuler2018nearlineartimeapproximationalgorithms,agarwal2024soda}. More recently, some works have provided theoretical guarantees for Wasserstein barycenters in the $p = 2$ setting: \cite{altschuler21barycenter,agarwal25soda} showed that additive and multiplicative (respectively) approximations for Wasserstein barycenters can be computed in polynomial time for \emph{constant} dimension. Recently, \cite{boedihardjo2024private,he2023algorithmically} provided constructions for private measures that are close to input empirical measures over $[0,1]^d$ in 1-Wasserstein distance and \cite{feldman2024instance} provided instance-optimal constructions for finite metric spaces.

The Johnson-Lindenstrauss (JL) lemma \citep{Johnson1984ExtensionsOL} is a dimensionality reduction method that provides worst-case guarantees on preserving pairwise distances between a collection of points. It has been applied to numerous problems in many areas of computer science, including streaming algorithms \citep{alon1999tracking,muthukrishnan2005data} and DP \citep{blocki2012dp_johnson_lindenstrauss,nikolov2022dp_johnson_lindenstrauss}. The (lower) bound on the dimension required to approximately preserve solutions varies from problem to problem, e.g. see \cite{narayanan2021jl,charikar2025jl_approx} for a discussion. For facility location problems, \cite{makarychev2019performance} showed that dimension $d' = O(\log k)$ suffices for preserving the cost of solutions to $k$-means clustering, and \cite{izzo2021wasserstein_barycenter} showed that dimension $d' = O(\log n)$ suffices for Wasserstein barycenters supported on $\le n$ points. 
\section{Deferred algorithms}\label{app:deferred_algs}

\begin{algorithm}[h]
\begin{algorithmic}[1]
\REQUIRE{barycenter $\nu$, $k$ input distributions $\mu_1,\dots, \mu_k$ supported on $m_1, \dots, m_k$ points, respectively}
\STATE{Obtain $(\mathbf{S},\mathbf{w})$ based on $(T_1,\dots, T_k)$ as follows:}
\FOR{$i\in[k]$}
\STATE{$T_i\gets \mathsf{OT}(\mu_i,\nu)$}
\FOR{$(\ell,j)\in [m_i]\times [m]$}
\IF{$T_i[\ell,j] > 0$}
\STATE{$S_j \gets S_j \cup \{\mu_i[\ell]\}$}
\STATE{$w_j(\mu_i[\ell]) \gets T_i[\ell,j]$}
\ENDIF
\ENDFOR
\ENDFOR
\RETURN{$(\mathbf{S}, \mathbf{w})$}
\end{algorithmic}
\caption{$\mathsf{SolutionWeights}$}
\label{alg:private_weights}
\end{algorithm}

\begin{algorithm}[h]
\begin{algorithmic}[1]
\REQUIRE{$k$ discrete distributions $\mu_1,\dots, \mu_k$ supported on $\mathbb{R}^d$, partition $(\mathbf{S},\mathbf{w})$ as described in Definition \ref{def:solution}}
\FOR{$S_j\in \mathbf{S}$}
\STATE{$\nu^{(j)}\gets \arg\min\sum_{x\in S_j}w_j(x)\|x-\nu^{(j)}\|^p$}
\ENDFOR
\RETURN{$(\nu^{(1)},\dots, \nu^{(m)})$}
\end{algorithmic}
\caption{$\mathsf{SupportPoints}$}
\label{alg:find_support}
\end{algorithm}

\begin{algorithm}[h]
\begin{algorithmic}[1]
\renewcommand{\COMMENT}[1]{%
  \hfill\makebox{$\triangleright$~~#1}}

\REQUIRE{$k$ discrete distributions $\mu_1,\dots, \mu_k$ supported on $\mathbb{R}^d$, subsampling parameter $k'$, approximate Wasserstein barycenter algorithm $\mathsf{Approx}$, privacy parameter $\epsilon, \delta$}
\FOR{$i\in [n]$}
\STATE{Shuffle $\mu_i$}
\STATE{$\mu_{i,j} \gets \mu_i\left[j\cdot \lfloor\frac{n}{k'}\rfloor : (j+1)\lfloor \frac{n}{k'}\rfloor\right]~~\forall j\in[k']$}
\ENDFOR
\STATE{$\nu:=(\nu^{(1)}, \dots, \nu^{(m)})\gets \mathsf{Approx}^\mathsf{JL}(\mu_{1,1},\dots, \mu_{1,k'},\cdots, \mu_{k,1},\cdots, \mu_{k,k'})$}
\STATE{$(\Tilde{\nu}^{(1)}, \dots, \Tilde{\nu}^{(m)})\gets (\nu^{(1)}, \dots, \nu^{(m)}) + \mathcal{N}(0, \sigma^2\cdot I_{md})$, where \[
\sigma^2 := \frac{2m\log(1.25/\delta)}{\left(\epsilon kk'\right)^2}. 
\]}
\RETURN{$\Tilde{\nu}$ with uniform support on $\{\tilde{\nu}^{(j)}\}_{j\in[m]}$}
\end{algorithmic}
\caption{$\mathsf{OutputPerturbationBarycenterSubsampled}^{\epsilon,\delta}$}
\label{alg:wb_output_perturbation_subsampled}
\end{algorithm}

\section{Additional preliminaries and lemmata}\label{app:lemmata}
\subsection{Differential privacy}
\begin{lemma}[Parallel composition]\label{lem:parallel_comp}
    Let $\mathcal{A}_1,\dots, \mathcal{A}_k$ be $\epsilon$-DP algorithms. Suppose $\mathcal{D} = S_1\cup \cdots \cup S_k$, where $S_i \cap S_j = \emptyset$ for every $i \neq j$. Then $(\mathcal{A}_1(S_1),\dots, \mathcal{A}_k(S_k))$ is $\epsilon$-DP.  
\end{lemma}

A nice property of differential privacy is the post-processing property, which informally says that transforming private output does not incur additional privacy loss. Formally, we have the following:
\begin{lemma}[Post-processing]\label{lem:post_processing}
    Let $\mathcal{A}$ be an $\epsilon$-DP algorithm. Then for any (possibly randomized algorithm) $g$, $g\circ \mathcal{A}(\mathcal{D})$ is $\epsilon$-DP. 
\end{lemma}

\begin{definition}[$\ell_p$-sensitivity]
We define the $\ell_p$-sensitivity of a function $f$ to be \[
\Delta_p f:= \max_{\mathcal{D}, \mathcal{D}'}\|f(\mathcal{D}) - f(\mathcal{D}')\|_p,
\]
where $\mathcal{D}, \mathcal{D}'$ are adjacent datasets.
\end{definition}

\begin{lemma}[Gaussian mechanism]\label{lem:gaussian_mech}
Let $f$ be a function, $\epsilon,\delta \in (0, 1)$, and $\sigma^2 \ge \left(\Delta_2f\right)^2\cdot\frac{2\log(1.25/\delta)}{\epsilon^2}$. The Gaussian mechanism $f(\mathcal{D})+\mathcal{N}(0, \sigma^2)$ is $(\epsilon,\delta)$-DP. 
\end{lemma}

\subsection{Optimal transport}
It can easily be checked that indeed the Wasserstein distance is a metric: in particular, the triangle inequality holds.
\begin{lemma}[{\cite{santambrogio2015optimal}, Lemma 5.4}]\label{lem:triangle_inequality}
    For any $p \ge 1, \mu,\nu,\pi\in \mathcal{P}_p(\mathcal{X})$, we have $W_p(\mu,\pi) \le W_p(\mu,\nu)+ W_p(\nu,\pi)$.
\end{lemma}

For bounded spaces, we can bound the $p$-Wasserstein distance by the $1$-Wasserstein distance:
\begin{lemma}[{\cite{santambrogio2015optimal}}]\label{lem:reverse_comp}
    Let $\mathcal{X}$ be bounded. Then for any $p \ge 1, \mu,\nu\in \mathcal{P}_p(\mathcal{X})$, we have $W_p(\mu,\nu) \le \operatorname{diam}(\mathcal{X})^{(p-1)/p}W_1(\mu,\nu)^{1/p}$.
\end{lemma}

\subsection{JL transform}
A JL transform is any (linear) map that satisfies the JL lemma:
\begin{theorem}[Johnson-Lindenstrauss lemma, {\cite{Johnson1984ExtensionsOL}}]\label{lem:jl}
    Given an accuracy parameter $0 < \gamma < 1$, a set of $n$ points $X$ in $\mathbb{R}^d$, and the projection dimension $d' = O(\log n /\gamma^2)$, there exists a linear map $\Pi:\mathbb{R}^d\to \mathbb{R}^{d'}$ such that all pairwise distances are preserved within factor $(1\pm \gamma)$, i.e., it holds \[
    \frac{1}{1+\gamma}\|x-y\|\le \|\Pi x-\Pi y\|\le (1+\gamma)\|x-y\|
    \] for every $x,y\in X$, with probability $1 - 1/{\operatorname{poly}(n)}$.
\end{theorem}
In other words, the JL transform reduces the dimensionality of the data from $d$ to $d'$ while approximately preserving all pairwise distances (whp).
Note that this worst-case guarantee is the main strength of the JL approach, 
other dimensionality reduction techniques such as principal component analysis typically do not have this guarantee.

\subsection{Useful lemmata}
\begin{lemma}\label{lem:wp_distance_noisy}
    Let $\mu\in\mathcal{P}_p(\mathbb{R}^d)$ and $\mu_\sigma := \mu * \mathcal{N}(0, \sigma^2I_d)$. Then \[
    W_p(\mu, \mu_\sigma)\lesssim \sigma\left(\sqrt{d} + \sqrt{2p}\right).
    \]
\end{lemma}
\begin{proof}
    By definition of Wasserstein distance, we have \[
    W_p(\mu,\mu_\sigma) \le \left(\mathbb{E}[X - (X+\sigma Z)]^p\right)^{1/p} = \sigma \left(\mathbb{E}\|Z\|^p\right)^{1/p},
    \]
    where $X\sim \mu$ and $Z\sim\mathcal{N}(0, I_d)$. We have \begin{align*}
        \mathbb{E}\|Z\|^p &= \mathbb{E}\left(\|Z\|^2\right)^{p/2} = \mathbb{E}[Y^{p/2}],
    \end{align*}
    where $Y := \|Z\|^2\sim \chi_d^2$ (chi-squared with $d$ degrees of freedom). Using \cite{laurent2000adaptive}, it holds \begin{align*}
        \Pr[Y - d \ge 2 \sqrt{dt}+2t]&\le \exp(-t)\\
        \Pr[d - Y \ge 2\sqrt{dt}]&\le \exp(-t).
    \end{align*}
    Applying Theorem 2.3 of \cite{boucheron2003concentration} yields the desired result.
\end{proof}

\begin{lemma}\label{lem:better_minkowski}
    For every $p\ge 1$, if $0\le a, b \le 1$, then it holds \[
    (a + b)^p \le a^p + p(a + b)^{p-1}b \le a^p + p2^{p-1}b.
    \]
\end{lemma}
\begin{proof}
Let $g(t) = (a+t)^p$ for $t\in[0,b]$. Observe that $g$ is convex as $g''(t) = p(p-1)(a+t)^{p-2}\ge0$. Thus, by convexity, we have \[
    (a+b)^p - a^p = g(b) - g(0)\le g'(b)b = p(a+b)^{p-1}b.
    \]
    Rearranging yields the first inequality.  The second inequality is immediate by $0\le a,b\le 1$. This concludes the proof.
\end{proof}

\begin{lemma}[Hoeffding's inequality]\label{lem:hoeffding}
    Let $X_1,\dots, X_n$ be i.i.d. random variables such that $X_i \in [a_i,b_i]$ a.s. Let $S_n := \sum_{i\in[n]}X_i$. Then \[
    \Pr[|S_n-\mathbb{E}[S_n]|\ge t]\le 2\exp\left(-\frac{2t^2}{\sum_{i\in[n]}(b_i-a_i)^2}\right).
    \]
\end{lemma}

\section{Comparing clustering and Wasserstein barycenters}\label{app:clustering_vs_wasserstein}

\subsection{Clustering}
For completeness, we provide a brief discussion of the construction for private clusterings. We start by formalizing the problem statement for clustering:
\begin{definition}[$(k,p)$-clustering]
    Given $k\in\mathbb{N}$ and a dataset $\mathbf{X}= \{x_1,\dots, x_n\}$ in $B_0(1/2)$, we want to find $k$ centers $c_1,\dots, c_k\in\mathbb{R}^d$ that minimizes \begin{equation}
\cost_{\mathbf{X}}(c_1,\dots, c_k) := \sum_{i\in[n]}\left(\min_{j\in[k]}\|x_i - c_j\|\right)^p.   \label{eq:clustering_cost}     
    \end{equation}
\end{definition}
The optimal cost is denoted as $\OPT$, where we suppress the dependence on $k,p,\mathbf{X}$.
\begin{definition}[Approximation for clustering]
    A $(w,t)$-approximation algorithm for $(k,p)$-clustering outputs $c_1,\dots, c_k$ such that $\cost(c_1,\dots, c_k) \le w\cdot \OPT + t$.
\end{definition}

\cite{ghazi2020dp_clustering} showed the following:
\begin{theorem}
    For any $p\ge 1$, suppose that there exists a polynomial time (not necessarily private) $(w, 0)$-approximation algorithm for the $(k,p)$-clustering problem. Then, for every $\epsilon>0$ and $ \delta,\gamma,\xi\in(0,1)$, there exists an $(\epsilon,\delta)$-DP algorithm that runs in $\left(\frac{k}{\xi}\right)^{O_{p,\gamma}(1)}\cdot \operatorname{poly}(nd)$ time and
    outputs an \[
    \left(w(1+\gamma), O_{p,\gamma,w}\left(\frac{k\sqrt{d}}{\epsilon}\cdot \operatorname{poly}\log\frac{k}{\delta\gamma} + \frac{(k/\gamma)^{O_{p,\gamma}(1)}}{\epsilon}\cdot \operatorname{poly}\log\frac{n}{\gamma}\right)\right)\]
    -approximation for $(k,p)$-clustering, with probability $1-\xi$.
\end{theorem}    

The first term in the additive error comes from computing the centers in high dimension. The second term comes from bounding the error in low dimension. 

The construction of private $(k,p)$-clustering in low dimension works as follows. Under DP, we start by finding a centroid set\footnote{A set that contains a good approximate solution.} of size $O(k\log n)$ with small multiplicative error. Then, we create a private coreset by ``snapping'' all the input data to their nearest point in the centroid set and use the (discrete) Laplace mechanism to privatize the counts. 

\subsection{Stability}\label{app:robustness}
The $(k,p)$-clustering objective is stable in the following sense: suppose we fix a dataset $\mathbf{X}$ and $k$ centers. If we move one datapoint and update the $k$ centers, a large fraction of the points will remain clustered together. This fact is key to the accuracy of private clustering algorithms.

Due to the stability of the clustering objective, the snapping procedure above incurs small additive error and is easy to reason about. On the other hand, optimal transport plans are highly non-stable:

\begin{proposition}\label{prop:counterexample}
Let $n\in\mathbb{N}$. Fix a distribution $\nu$ supported on $n$ atoms with uniform weights. There exists a distribution over $\mathbb{R}$ such that moving one datapoint by $O\left(\frac{1}{n}\right)$ changes the mapping for every datapoint.
\end{proposition}
\begin{proof}
    We construct $\mu_n$ with support over $\left\{\frac{k}{n}\right\}_{k\in[n]}$. Recall that the optimal transport plan in one dimension can be computed  from the cumulative density function, e.g. see \cite{santambrogio2015optimal}. In the setting of the proposition, this will just be based on the order of the datapoints. Thus moving the particle on $\frac{1}{n}$ to $\frac{2 + \epsilon}{n}$ for any $\epsilon > 0$ will yield the desired result.
\end{proof}
\begin{remark}
The distribution in Proposition \ref{prop:counterexample} also shows that in order to use the private coreset construction from $k$-means clustering for Wasserstein distances requires taking $k = \Tilde{\Omega}(n)$ for small snapping error, when we usually should think of $k = \Tilde{O}(1)$.  
\end{remark}

\begin{proposition}\label{prop:counterexample2}
    There exists a distribution over $\mathbb{R}^2$ such that changing one point by $O(1)$ causes all the points in the support of the output barycenter by $\Omega(1)$.\footnote{whereas the expected should be $O\left(\frac{m}{n}\right)$}
\end{proposition}
\begin{proof}
    See Figure \ref{fig:counterexample}.
\end{proof}

\begin{figure}[h]
    \centering
    \begin{subfigure}[t]{0.25\textwidth}
    \centering
    \includegraphics[width=\textwidth]{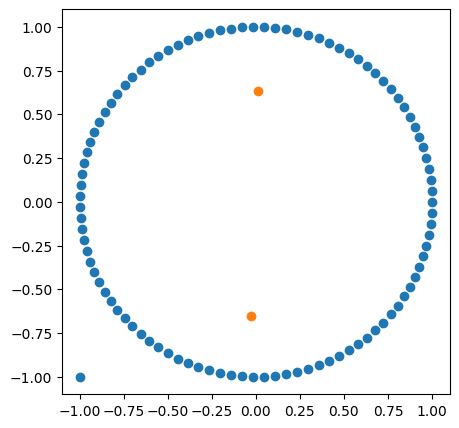}
    \caption{Non-optimal barycenter.}
    \end{subfigure}\hspace{0.4in}
    \begin{subfigure}[t]{0.25\textwidth}
    \centering
    \includegraphics[width=\textwidth]{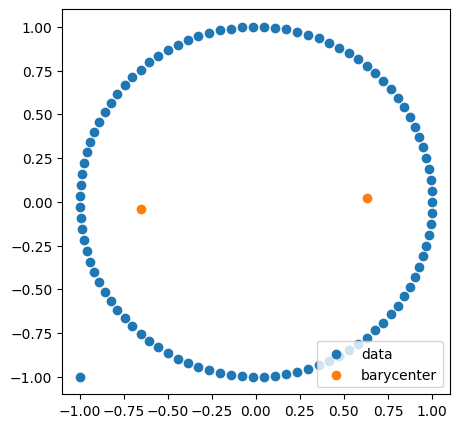}
    \caption{Optimal barycenter.}
    \end{subfigure}
    \caption{Unperturbed data is uniform over $\mathbb{S}^1$. Here, the averages of any of the two disjoint half-arcs yield an optimal barycenter. However, with a bad initialization, each point in the support of the output distribution can move $\Omega(1)$ as $\Omega(n)$ of the couplings change.}
    \label{fig:counterexample}
\end{figure}

\section{Deferred proofs}\label{app:def_proofs}

\subsection{Proof of Lemma \ref{lem:wass_distance_to_coreset}}\label{app:lem_wass_distance_to_coreset}
\begin{lemma}[Lemma {\ref{lem:wass_distance_to_coreset}}, restated]
    Let $\mu_1,\dots, \mu_k$ be discrete probability measures and suppose $\mu_1',\dots, \mu_k'$ are $(p, 1,t)$-coresets for each $\mu_i$, respectively. Then,
    \[
    \OPT_{(\mu_1',\dots, \mu_k')}\le  \OPT_{(\mu_1,\dots,\mu_k)}+O_p(t^p).
    \]
\end{lemma}

\begin{proof}
Consider a candidate barycenter $\nu$. Using \eqref{eq:wb_objective}, we have \begin{align}
    \cost_{\mu_1',\dots, \mu_k'}(\nu) &= \frac{1}{k}\sum_{i=1}^k W_p^p(\mu_i',\nu)\notag \\
    &\le \frac{1}{k}\sum_{i=1}^k\left(W_p(\mu_i',\mu_i) + W_p(\mu_i,\nu)\right)^p\label{eq:45i}\\
    &\le \frac{1}{k}\sum_{i=1}^k W_p^p(\mu_i,\nu) + \frac{p2^{p-1}}{k}\sum_{i=1}^k W_p^p(\mu_i',\mu_i)\label{eq:45ii}\\
    &= \cost_{\mu_1,\dots, \mu_k}(\nu) + \frac{p2^{p-1}}{k}\sum_{i=1}^k W_p^p(\mu_i',\mu_i)\label{eq:45iii}\\
    &\le \cost_{\mu_1,\dots, \mu_k}(\nu) + \frac{p2^{p-1}}{k}\sum_{i=1}^k t^p,\label{eq:45iv}\\
    &= \cost_{\mu_1,\dots, \mu_k}(\nu) + O_p(t^p)\notag
\end{align}
where \eqref{eq:45i} follows from the triangle inequality, \eqref{eq:45ii} follows from Lemma \ref{lem:better_minkowski}, \eqref{eq:45iii} follows from \eqref{eq:wb_objective}, and \eqref{eq:45iv} follows from Definition \ref{def:coreset_wasserstein_distance}. Note that applying Lemma \ref{lem:better_minkowski} uses Assumption \ref{asmp:supp}. This concludes the proof.
\end{proof}

\subsection{Proof of Theorem \ref{thm:main}}\label{app:thm_main}

\begin{theorem}[Theorem \ref{thm:main}, restated]
    For any $p\ge 1$, suppose that there exists a (not necessarily private) $(z, t)$-approximation algorithm that runs in time $2^{O(d)}\cdot \operatorname{poly}(n,k)$ for the $p$-Wasserstein barycenter problem. Then, for every $\epsilon>0$ and $ \gamma,\xi\in(0,1)$, there exists a polynomial-time $\epsilon$-DP algorithm that outputs an 
    \[
    \left(z(1+\gamma), O_{p,\gamma,z}\left(\frac{1}{(\epsilon n)^{1/d}}\cdot\operatorname{poly}\log\left(\frac{k}{ \xi}\right) + t\right)\right)\]
    -approximate $p$-Wasserstein barycenter, with probability $1-\xi$.
\end{theorem}
\begin{proof}
Let $d'$ be as in Theorem \ref{thm:izzo_main}.

(Runtime) It suffices to bound the runtime of computing the barycenter in low dimensions as it is clear that the pre- and post-processing steps run in polynomial time. With the given $d'$, we have that $2^{O(d')}\cdot \operatorname{poly}(n,k) = \operatorname{poly}(n)\cdot \operatorname{poly}(n,k) = \operatorname{poly}(n,k)$, as desired.

(Privacy) The privacy follows from Lemmas \ref{lem:parallel_comp} and \ref{lem:post_processing} as the input distributions are disjoint.

(Utility) For each $\mu_i$, we invoke Theorem \ref{thm:wasserstein_distance_coreset} with failure probability $\frac{\xi}{2k}$. Then by a union bound, with probability $1 - \frac{\xi}{2}$, $\mu_i'$ is a $\left(1,O_p\left(\left(\frac{1}{(\epsilon n)^{1/d}}\cdot\operatorname{poly}\log\left(\frac{k}{ \xi}\right)\right)^{1/p}\right)\right)$-coreset for $\mu_i$, for each $i \in [k]$. Assuming this event holds, Lemma \ref{lem:wass_distance_to_coreset} implies 
\begin{equation}
    \OPT_{(\mu_1',\dots, \mu_k')}\le \OPT_{(\mu_1,\dots, \mu_k)} + O_p\left(\frac{1}{(\epsilon n)^{1/d}}\cdot\operatorname{poly}\log\left(\frac{k}{ \xi}\right)\right).\label{eq:c2i}
\end{equation}
Let $\nu$ be the output of the algorithm. Now we apply Theorem \ref{thm:izzo_main}, along with the guarantee of the not necessarily private approximation algorithm, which implies with probability $1 - \frac{\xi}{2}$,  \begin{align}
\cost_{(\mu_1',\dots,\mu_k')}(\nu) &\le z(1+\gamma)\OPT_{(\mu_1',\dots,\mu_k')} + O_{p, \gamma,z}(t)\label{eq:c2ii}
\end{align}
By a union bound, \eqref{eq:c2i} and \eqref{eq:c2ii} both occur with probability $1 - \xi$. When this is the case, we deduce \begin{align*}
    \operatorname{cost}_{(\mu_1',\dots, \mu_k')}(\nu) \le z(1+\gamma)\OPT_{(\mu_1,\dots, \mu_k)} + O_{p,\gamma,z}\left(\frac{1}{(\epsilon n)^{1/d}}\cdot\operatorname{poly}\log\left(\frac{k}{ \xi}\right) + t\right),
\end{align*}
which concludes the proof.
\end{proof}

\subsection{Proof of Theorem \ref{prop:wb_output_perturbation}}\label{app:thm_wb_output_perturbation}
\begin{theorem}[{Theorem \ref{prop:wb_output_perturbation}, restated}]
 For any $p\ge 1$, suppose that there exists a (not necessarily private) $(z, t)$-approximation algorithm that runs in time $2^{O(d)}\cdot \operatorname{poly}(n,k)$ for the $p$-Wasserstein barycenter problem. Then, for every $\epsilon>0$ and $\delta,\gamma,\xi \in (0, 1)$, there exists a polynomial-time $(\epsilon,\delta)$-DP algorithm that outputs an \[
\left(z(1+\gamma), \Tilde{O}_{p,\gamma,z,\xi}\left(\left(\frac{md\log (1/\delta)}{(\epsilon k)^2}\right)^{p/2}\right) + t\right)\]
-approximation for the $p$-Wasserstein barycenter problem, with probability $1 - \xi$. 
\end{theorem}

\begin{proof} We prove the privacy and utility separately.

(Privacy) Consider the $\ell_2$ sensitivity of the algorithm, which is a function $\mathcal{X}^{nk}\to \mathbb{R}^{m\times d}$. If we change one datapoint, the couplings of up to $n$ elements could potentially change, namely all of those in the subpopulation. By normalization and Assumption \ref{asmp:supp}, this implies the $\ell_2$ sensitivity of $\|\nu^{(j)} - \nu'^{(j)}\|_2$ is $\frac{1}{k}$. Thus, the $\ell_2$ sensitivity of the output is \[
\|\nu_1\circ\cdots\circ \nu_m - \nu_1'\circ\cdots\circ\nu_m'\|_2 = \left(\sum_{j\in[m]}\|\nu_j-\nu_j'\|^2\right)^{1/2} \le \left(m\left(\frac{1}{k}\right)^2\right)^{1/2} = \frac{\sqrt{m}}{k},
\] 
where $\nu_1\circ\cdots \circ \nu_m\in\mathbb{R}^{md}$ is vector-concatenation. Privacy follows by the guarantees of the Gaussian mechanism (Lemma \ref{lem:gaussian_mech}).

(Utility) We have \begin{align}
    \cost_{\mu_1,\dots, \mu_k}(\Tilde{\nu}) &= \frac{1}{k}\sum_{i = 1}^k W_p^p(\mu_i, \Tilde{\nu})\notag\\
    &\le \frac{1}{k}\sum_{i=1}^k(W_p(\mu_i,\nu) + W_p(\nu, \Tilde{\nu}))^p\label{eq:simple_i}\\
    &\le \frac{1}{k}\sum_{i=1}^kW_p^p(\mu_i,\nu) + p2^{p-1}W_p^p(\nu, \Tilde{\nu})\label{eq:simple_ii}\\
    &= \operatorname{cost}_{\mu_1,\dots,\mu_k}(\nu) + p2^{p-1}W_p^p\left(\nu, \nu*\mathcal{N}\left(0, \frac{2m\log(1.25/\delta)}{\left(\epsilon k\right)^2}I_d\right)\right)\notag\\
    &\le \operatorname{cost}_{\mu_1,\dots,\mu_k}(\nu) +p2^{p-1}\left(\frac{2m\log(1.25/\delta)}{(\epsilon k)}\right)^{p/2}\left(\sqrt{d} + \sqrt{2p}\right)^p\label{eq:simple_iii}\\
    &\le \operatorname{cost}_{\mu_1,\dots,\mu_k}(\nu) + O_{p}\left(\left(\frac{md\log (1/\delta)}{(\epsilon k)^2}\right)^{p/2}\right),\notag
\end{align}
where \eqref{eq:simple_i} follows from Lemma \ref{lem:triangle_inequality}, \eqref{eq:simple_ii} follows from Lemma \ref{lem:better_minkowski}, and \eqref{eq:simple_iii} follows from Lemma \ref{lem:wp_distance_noisy}. Then, the claim follows from an analogous application dimensionality reduction as in Appendix \ref{app:thm_main}.
\end{proof}

\subsection{Proof of Proposition \ref{thm:subsampled_output_perturbation}}\label{app:thm_subsampled_output_perturbation}
Before we provide the proof, we introduce some preliminaries. 
\begin{definition}
    The total variation distance between two discrete measures $\mu,\nu$ is defined to be\[
\|\mu-\nu\|_\mathsf{TV} = \frac{1}{2}\sum_x|\mu(x) - \nu(x)|.
    \]
\end{definition}
\begin{proposition}\label{prop:wp_bounded_by_tv}
    Let $(\mathcal{X}, c)$ be a metric space with diameter $D$. For any $p\ge 1$ and probability measures $\mu,\nu\in \mathcal{P}(\mathcal{X})$, it holds \[
    W_p(\mu,\nu)\le D\|\mu - \nu\|_\mathsf{TV}^{1/p}.
    \]    
\end{proposition}
\begin{proof}
    Let $\omega = \mu \land \nu$, and rewrite $\mu = \omega  +\alpha$ and $\nu = \omega + \beta$ for $\alpha,\beta\ge 0$ and $\alpha(\mathcal{X})=\beta(\mathcal{X})=t = \|\mu - \nu\|_\mathsf{TV}$. Now we construct a coupling $\pi\in\Pi(\mu,\nu)$ as follows. We couple $\mu$ and $\nu$ by coupling $\omega$ to itself and $\alpha$ to $\beta$ with any plan $\gamma\in \Pi(\alpha,\beta)$. This yields \[
    \int c(x,y)^pd\pi(x,y) \le \int c(x,y)^pd\gamma(x,y)\le D^p\|\mu - \nu\|_\mathsf{TV},
    \]
    where we use $c(x,y)\le D$. Taking the infimum over couplings and taking the $p$th root yields the desired result.
\end{proof}

\begin{proposition}\label{prop:subsample_tv}
Let $X :=\{X_i\}_{i\in [n]}\subset \mathcal{X}$, which forms measure $\mu_n = \frac{1}{n}\sum_{i\in[n]}\delta_{X_i}$, and let $S\subset[n]$ be a subset of $[n]$ of size $m$, chosen uniformly at random, without replacement. Form $\mu_m := \frac{1}{m}\sum_{j\in S}\delta_{X_j}$. It holds \[
\|\mu_n - \mu_m\|_\mathsf{TV} \le 1 - \frac{m}{n}.
\]
\end{proposition}
\begin{proof}
    We have \begin{align*}
        \|\mu_n - \mu_m\|_\mathsf{TV} &= \frac{1}{2}\sum_k \left|\frac{m_k}{m} - \frac{n_k}{n}\right|\\
        &\le \frac{1}{2}\sum_k\sum_{i:X_i=x^{(k)}}\left|\frac{\mathbf{1}_{\{i\in S\}}}{m}-\frac{1}{n}\right|\\
        &= \frac{1}{2}\left(\frac{1}{m}-\frac{1}{n}\right)\sum_k m_k + \frac{1}{2n}\sum_k(n_k-m_k)\\
        &= 1 - \frac{m}{n},
    \end{align*}
    where we group the samples by distinct points $x^{(k)}$ with multiplicities $n_k$ (and similarly $m_k$ for the subsample). This concludes the proof.
\end{proof}

\begin{corollary}\label{cor:w2_subsampling}
    Let $\mu_n$ and $\mu_m$ be in Proposition \ref{prop:subsample_tv}, and suppose $\mathsf{diam}(\mathcal{X})\le D$. Then \[
    W_p(\mu_n,\mu_m)\le D\left(1 - \frac{m}{n}\right)^{1/p}.
    \]
\end{corollary}
\begin{proof}
This is a straightforward consequence of Proposition \ref{prop:wp_bounded_by_tv} and \ref{prop:subsample_tv}.
\end{proof}

Now we provide the proof of Theorem \ref{thm:subsampled_output_perturbation}, which we restate for convenience.
\begin{theorem}[Theorem \ref{thm:subsampled_output_perturbation}, restated]
By splitting each distribution into $k'$ disjoint distributions, in the same setting as Theorem \ref{prop:wb_output_perturbation}, we have an
\[
\left(z(1+\gamma), \Tilde{O}_{p,\gamma,z,\xi}\left(\left(\frac{md\log (1/\delta)}{(\epsilon kk')^2}\right)^{p/2} + \left(1 - \frac{1}{k'}\right) + t\right)\right)
\]
-approximation for the $p$-Wasserstein barycenter problem, with probability $1- \xi$.
\end{theorem}

\begin{proof}
Privacy is immediate by parallel composition, Lemma \ref{lem:parallel_comp}, and the privacy guarantees of Theorem \ref{prop:wb_output_perturbation}. For utility, we mirror the analysis of Theorem \ref{prop:wb_output_perturbation} to obtain
\begin{align*}
     \cost_{\mu_{1, 1}, \dots, \mu_{k,k'}}(\Tilde{\nu}) &= \frac{1}{k k'}\sum_{i = 1}^{k}\sum_{j=1}^{k'} W_p^p(\mu_{i,j}, \Tilde{\nu})\\
    &\le \frac{1}{kk'}\sum_{i=1}^k\sum_{j=1}^{k'}(W_p(\mu_{i,j},\nu) + W_p(\nu, \Tilde{\nu}))^p\\
    &\le \frac{1}{kk'}\sum_{i=1}^k\sum_{j=1}^{k'}W_p^p(\mu_{i,j},\nu) + p2^{p-1}W_p^p(\nu, \Tilde{\nu})\\
    &\le \frac{1}{kk'}\sum_{i=1}^k\sum_{j=1}^{k'}W_p^p(\mu_{i,j},\nu) + p2^{p-1}W_p^p(\nu, \Tilde{\nu})\\
    &\le \frac{1}{k}\sum_{i=1}^k\frac{1}{k'}\sum_{j=1}^{k'}\left(W_p(\mu_{i,j}, \mu_i) +W_p(\mu_{i},\nu)\right)^p + p2^{p-1}W_p^p(\nu, \Tilde{\nu})\\
    &\le \frac{1}{k}\sum_{i=1}^kW_p^p(\mu_i,\nu) + 
    \frac{p2^{p-1}}{kk'}\sum_{i=1}^k\sum_{j=1}^{k'}W_p^p(\mu_{i,j}, \mu_i) + p2^{p-1}W_p^p(\nu, \Tilde{\nu})\\
    &\le \frac{1}{k}\sum_{i=1}^kW_p^p(\mu_i,\nu) + 
p2^{p-1}\left(1 - \frac{1}{k'}\right) + p2^{p-1}W_p^p(\nu, \Tilde{\nu})\\
    &=\operatorname{cost}_{\mu_1,\dots,\mu_k}(\nu)+p2^{p-1}\left(1 - \frac{1}{k'}\right) + p2^{p-1}W_p^p\left(\nu, \nu*\mathcal{N}\left(0, \frac{2m\log(1.25/\delta)}{\left(\epsilon kk'\right)^2}I_d\right)\right)\\
    &\le \operatorname{cost}_{\mu_1,\dots,\mu_k}(\nu)+p2^{p-1}\left(1 - \frac{1}{k'}\right) +p2^{p-1}\left(\frac{2m\log(1.25/\delta)}{(\epsilon kk')}\right)^{p/2}\left(\sqrt{d} + \sqrt{2p}\right)^p\\
    &\le \operatorname{cost}_{\mu_1,\dots,\mu_k}(\nu) + O_{p}\left(\left(\frac{md\log (1/\delta)}{(\epsilon kk')^2}\right)^{p/2} + \left(1 - \frac{1}{k'}\right)\right),
\end{align*}
where we use Lemma \ref{lem:better_minkowski} and Corollary \ref{cor:w2_subsampling}.  Then, the claim follows from an analogous application dimensionality reduction as in Appendix \ref{app:thm_main}.
\end{proof}

\subsection{Proof of Proposition \ref{prop:approx_clusterable}}\label{app:prop_approx_clusterable}
Here, we state the full result.
\begin{proposition}\label{prop:approx_clusterable'}
    If $\mu$ is $(m,\Delta, c)$-approximately clusterable, then for all $n \le m(2\Delta)^{-2p}$, letting $\hat{\mu}_n = \frac{1}{n}\sum_{i\in[n]}\delta_{X_i}$ for $X_i\sim \mu$ i.i.d., it holds \begin{align*}
    W_p^p(\mu, \hat{\mu}_n)\le (1-c)\left((9^p+3)\sqrt{\frac{m}{(1-c-r)n}} +\sqrt{\frac{\log(4/\xi)}{2(1-c-r)n}}\right)+c  + r
    \end{align*}
    with probability $ 1-\xi$, where $r = \sqrt{\frac{\log(4/\xi)}{2n}}$.
\end{proposition}
\begin{proof}
    Let $\hat{c}_n = \frac{1}{n}\sum_{i\in[n]}\mathbf{1}_{\{X_i \text{ is not clustered}\}}$, and let $\hat{\mu}_\mathsf{clustered}$ be the empirical measure of the $n_\mathsf{clustered}:=n(1-\hat{c}_n)$ clustered points. Using a coupling that transports all the clustered points to points in the clustered and matching the mismatched points and non-clustered points with cost $\le 1$ (since the diameter is $1$), we have 
    \begin{equation*}
    W_p^p(\mu,\hat{\mu}_n)\le (1-c)W_p^p(\mu_\mathsf{clustered},\hat{\mu}_{n,\mathsf{clustered}}) + c + |\hat{c}_n-c|. 
    \end{equation*}
    For the first term on the right-hand side, using Propositions 13 and 20 of \cite{weed2019asymptotic}, for any $n_\mathsf{clustered}\le m(2\Delta)^{-2p}$, we have \[
    \mathbb{E}[W_p^p(\mu_\mathsf{clustered},\hat{\mu}_{n,\mathsf{clustered}})]\le (9^p+3)\sqrt{\frac{m}{n}}
    \]
    and 
    \begin{equation}
\Pr[W_p^p(\mu_\mathsf{clustered},\hat{\mu}_{n,\mathsf{clustered}})\ge  \mathbb{E}[W_p^p(\mu_\mathsf{clustered},\hat{\mu}_{n,\mathsf{clustered}})] + t]\le \exp(-2n_\mathsf{clustered}t^2).\label{eq:eq:approx_clusterable_i}
    \end{equation}    
Let $r := \sqrt{\frac{\log(4/\xi)}{2n}}$. Using Hoeffding's inequality, Lemma \ref{lem:hoeffding}, with probability at least $1 - \frac{\xi}{2}$, $|\hat{c}_n-c|\le r$, which also implies $n_\mathsf{clustered} = n(1-\hat{c}_n)\ge (1-c - r)n$. Using $r=t$ for \eqref{eq:eq:approx_clusterable_i}, we have with probability at least $1-\frac{\xi}{2}$, \[
W_p^p(\mu_\mathsf{clustered},\hat{\mu}_{n,\mathsf{clustered}})\le (9^p + 3)\sqrt{\frac{m}{(1-c-r)n}}+\sqrt{\frac{\log(4/\xi)}{2(1-c-r)n}}.\]
The claim follows by a union bound.
\end{proof}

\subsection{Proof of Theorem \ref{thm:output_perturbation_best}}\label{app:thm_output_perturbation_best}
\begin{theorem}[Theorem \ref{thm:output_perturbation_best}, restated]
Suppose every input distribution $\{\mu_i\}_{i\in[k]}$ is $\left(O(m), \Delta, c\le \frac{1}{2}\right)$-approximately clusterable and $\frac{n}{k'}\le O(m)\cdot (2\Delta)^{-2p}$.  For any $p\ge 1$, suppose that there exists a (not necessarily private) $(z, t)$-approximation algorithm that runs in time $2^{O(d)}\cdot \operatorname{poly}(n,k)$ for the $p$-Wasserstein barycenter problem. Then, for every $\epsilon>0$ and $\delta,\gamma,\xi \in (0, 1)$, there exists an $(\epsilon,\delta)$-DP algorithm that outputs a \[
\left(z(1+\gamma), \Tilde{O}_{p,\gamma,z,\xi}\left(\left(\frac{md\log (1/\delta)}{(\epsilon kk')^2}\right)^{p/2}+\sqrt{\frac{mk'}{n}}+c + t\right)\right)
\]
 
    -approximation for the $p$-Wasserstein barycenter problem, with probability $1-\xi$.
\end{theorem}
\begin{proof}
Privacy is immedate as in Thereom \ref{thm:subsampled_output_perturbation}. For the utility analysis, using subsampled size $n' = \frac{n}{k'}$, independence of the data, $c \le \frac{1}{2}$,  and union bounding over the $kk'$ distributions, we have \[
W_p^p(\mu_{i,j},\mu_i)\lesssim \sqrt{\frac{m}{n'}} + c + \sqrt{\frac{\log(kk'/\xi)}{n'}},
\]
for every $(i,j)\in[k]\times[k']$. Then following the proof of Theorem \ref{thm:subsampled_output_perturbation} and using Proposition \ref{prop:approx_clusterable}, we deduce 
\begin{align*}
     \cost_{\mu_{1, 1}, \dots, \mu_{k,k'}}(\Tilde{\nu}) &\le \frac{1}{k}\sum_{i=1}^kW_p^p(\mu_i,\nu) + 
    \frac{p2^{p-1}}{kk'}\sum_{i=1}^k\sum_{j=1}^{k'}W_p^p(\mu_{i,j}, \mu_i) + p2^{p-1}W_p^p(\nu, \Tilde{\nu})\\
    &\le \frac{1}{k}\sum_{i=1}^kW_p^p(\mu_i,\nu) + 
\frac{p2^{p-1}}{kk'}\sum_{i=1}^k\sum_{j=1}^{k'}\left( \sqrt{\frac{mk'}{n}} + \sqrt{\frac{k'\log(kk'/\xi)}{n}}+c\right) + p2^{p-1}W_p^p(\nu, \Tilde{\nu})\\
    &\le \operatorname{cost}_{\mu_1,\dots,\mu_k}(\nu) + O_{p}\left(\left(\frac{md\log (1/\delta)}{(\epsilon kk')^2}\right)^{p/2} + \sqrt{\frac{mk'}{n}}+\sqrt{\frac{k'\log(kk'/\xi)}{n}}+c\right).
\end{align*}
Then, the claim follows from an analogous application dimensionality reduction as in Appendix \ref{app:thm_main}, by now union bounding over $2kk'$ events.
\end{proof}

\section{Experiments}\label{app:additional_experiments}
In all the experiments, we scale the noise down by $(240d)^{1/d}$, e.g. Figure \ref{fig:subsampling}. We do not use zero noise as it usually is unstable.

\begin{figure}[h]
    \centering
    \begin{subfigure}[t]{0.6\textwidth}
    \centering
    \includegraphics[width=\textwidth]{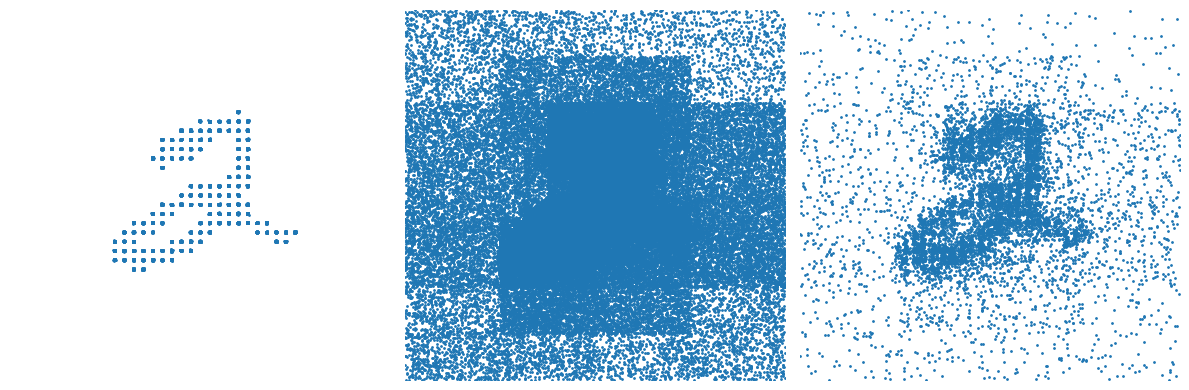}
    \caption{Points chosen uniformly at random from the full cell.}
    \label{fig:private_coreset_a}
    \end{subfigure}\\
    \begin{subfigure}[t]{0.6\textwidth}
    \centering
    \includegraphics[width=\textwidth]{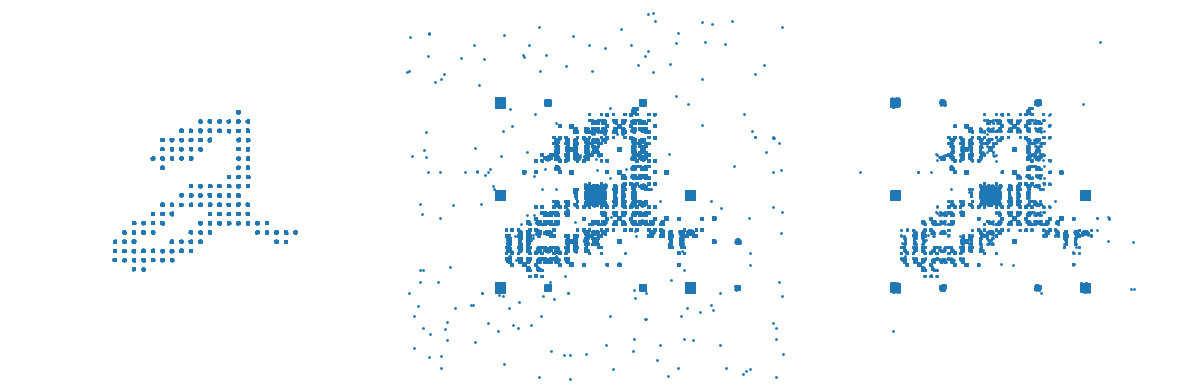}
    \caption{Points chosen uniformly at random from a scaled-down version of the cell.}
    \label{fig:private_coreset_b}
    \end{subfigure}
    \caption{Example of a MNIST 2 digit with privacy parameter $\epsilon = 1$. On the left, we have the sensitive data, upsampled to $n = 10000$. In the middle, we have the full private coreset (with $O(n\log \epsilon n)$ points). On the right, we choose a subsample the private coreset down to $n$ points.}
    \label{fig:private_coreset}
\end{figure}

\begin{figure}[b]
    \centering
    \begin{subfigure}[t]{0.5\textwidth}
    \centering
    \includegraphics[width=\textwidth]{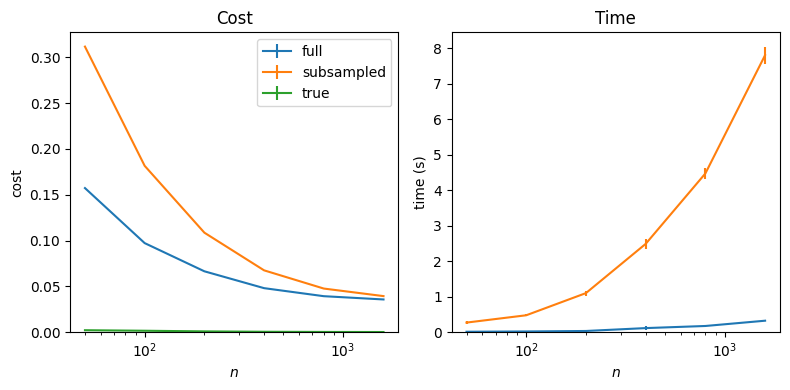}
    \caption{MNIST experiments with $n = 50, 100, \dots, 1600$ and $\epsilon = 1$, averaged over 10 runs. (left) Cost of solutions. (right) Runtime in seconds.}
    \label{fig:mnist}
    \end{subfigure}\hspace{0.2in}
    \begin{subfigure}[t]{0.25\textwidth}
    \centering
    \includegraphics[width=\textwidth]{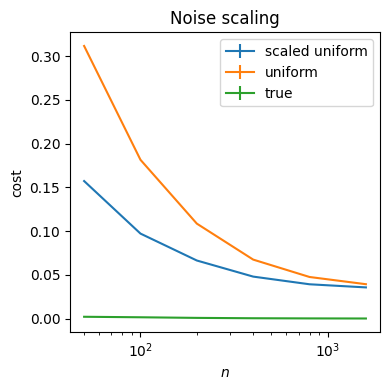}
    \caption{Uniform corresponds to the setting as Figure \ref{fig:private_coreset_a} and scaled uniform setup corresponds to the setting as Figure \ref{fig:private_coreset_b}.}
    \label{fig:subsampling}
    \end{subfigure}
    \caption{MNIST experiments testing sample size and noise scaling.}
    \label{fig:private_coreset}
\end{figure}

\subsection{Additional experiments}
We provide an example of the private coreset under uniform noise and scaled-down unifrom noise on the MNIST dataset (as points over $[-0.25, 0.25]^2$ in Figure \ref{fig:private_coreset}. In particular, for visualization, we treat one image as one distribution.

In the next experiment with MNIST, we follow the setup of \cite{izzo2021wasserstein_barycenter} and treat each image as a point in $\mathbb{R}^n$. We consider $k = 10$ (each digit is one distribution), $d' = 25$, and $m = 40$. One difference is that we pre-process the data onto $B_{0.5}(0)$. We consider the cost to be the Wasserstein distance between the output and the Wasserstein barycenter over 5000 images from each class. Our experiment parameters are $0.01$ for entropic regularization, $50$ iterations, and $100$ inner (Sinkhorn) iterations. In Figure \ref{fig:mnist}, we show that subsampling on the private coresets to size $n$ has a negligible increase in cost for sufficiently large $n$. In Figure \ref{fig:subsampling}, we observe that the choice of how points are sampled in each cell has a small effect on the cost for sufficiently large $n$.

In the US population experiment, we use GPT to implement the code for testing whether a point is inside the US population. 


\begin{figure}[h]
    \centering
    \includegraphics[width=0.6\linewidth]{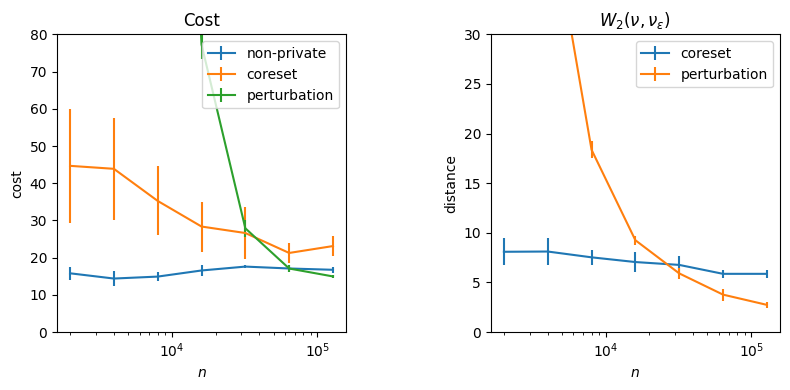}
    \caption{$n = 2000, 4000, \dots, 128000$ and $\epsilon = 5$, similar to Figure \ref{fig:us_population_a}, averaged over 10 trials.}
\end{figure}

\begin{figure}[h]
    \centering
    \begin{subfigure}{\textwidth}
    \centering
    \includegraphics[width=0.95\textwidth]{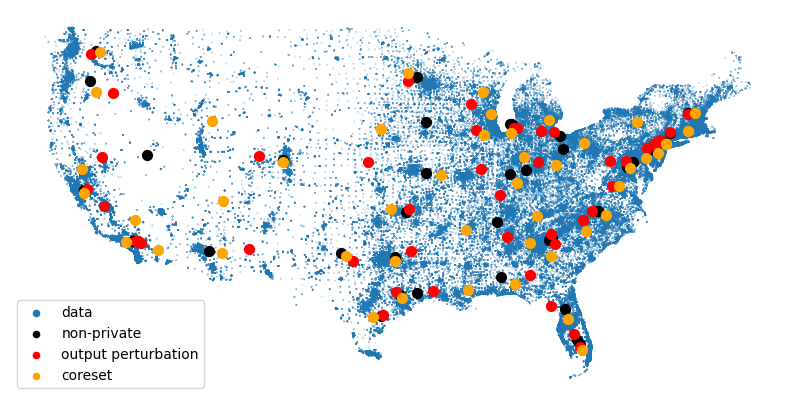}
    \caption{$n = 200000$ and $\epsilon = 5$ (and $\delta = \frac{1}{n}$) for $m = 48$ and $k = 1$. Denoting $\nu,\nu_\mathsf{core},\nu_\mathsf{pert}$ as the non-private, private coreset-based, and output-perturbation barycenters, respectively, we have $\mathrm{cost}(\nu) = 15.92$, $\mathrm{cost}(\nu_\mathsf{core}) = 24.99$, $\mathrm{cost}(\nu_\mathsf{pert}) = 16.957$ (squared degrees longitude/latitude),
    and $W_2(\nu,\nu_\mathsf{core}) = 5.964$, $W_2(\nu,\nu_\mathsf{pert}) = 2.906$ (degrees).}
    \end{subfigure}\\
    \begin{subfigure}{\textwidth}
    \centering
    \includegraphics[width=0.95\textwidth]{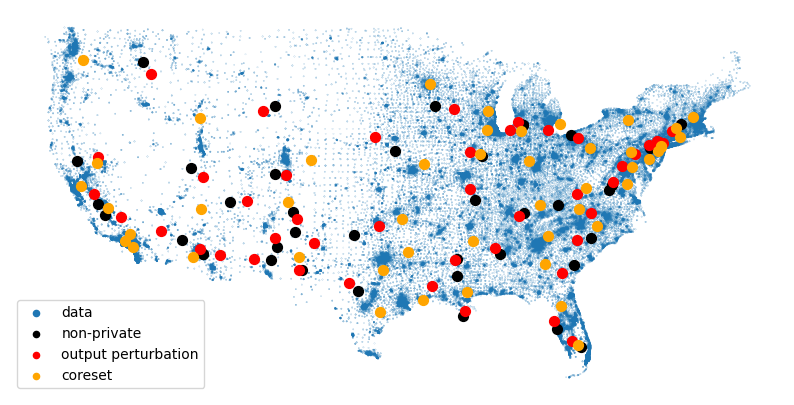}
    \caption{$n = 100000$ and $\epsilon = 5$ (and $\delta = \frac{1}{n}$) for $m = 48$ and $k = 4$ (self-reported White, Asian, Black, Hispanic). Denoting $\nu,\nu_\mathsf{core},\nu_\mathsf{pert}$ as the non-private, private coreset-based, and output-perturbation barycenters, respectively, we have $\mathrm{cost}(\nu) = 5018.618$, $\mathrm{cost}(\nu_\mathsf{core}) = 7046.907$, $\mathrm{cost}(\nu_\mathsf{pert}) = 4937.966$ (squared degrees longitude/latitude),
    and $W_2(\nu,\nu_\mathsf{core}) = 12.967$, $W_2(\nu,\nu_\mathsf{pert}) = 1.449$ (degrees).}
    \end{subfigure}
    \caption{Same experimental setup as Figure \ref{fig:us_population}, with $\epsilon = 5$.}
\end{figure}

\end{document}